\documentclass{article}
\usepackage{multicol}
\usepackage{microtype}
\usepackage{amsfonts}
\usepackage{caption}
\usepackage{tikz}
\usepackage[tight,footnotesize]{subfigure}
\usepackage{multirow}
\usepackage{mathtools}
\usetikzlibrary{shadows,fadings,patterns}
\usepackage{color}
\usepackage{booktabs} % for professional tables
\usepackage{amssymb,amsmath,amsthm}
\newtheorem{theorem}{Theorem}
\newtheorem{example}{Example}

\newtheorem{lemma}{Lemma}
\newtheorem{definition}{Definition}

\newtheorem{assumption}{Assumption}
\newcommand{\nosection}[1]{\vspace{2pt}\noindent\textbf{#1.}}
\usepackage{bm}
\newcommand{\E}{\mathbb{E}}
\newcommand{\R}{\mathbb{R}}
\newcommand{\rfdd}{DRFLM}
\newcommand{\w}{\mathbf{w}}
\newcommand{\x}{\boldsymbol{x}}
\newcommand{\Dset}{\mathcal{D}}
\newcommand{\z}{\boldsymbol{z}}
\newcommand{\blambda}{\boldsymbol{\lambda}}
\usepackage{hyperref}

\definecolor{highlight}{RGB}{102,205,170}

\usepackage[accepted]{icml2022}

\icmltitlerunning{DRFLM: Distributionally Robust Federated Learning
with Inter-client Noise via Local Mixup}
\begin{document}

\twocolumn[
\icmltitle{DRFLM: Distributionally Robust Federated Learning \\
with Inter-client Noise via Local Mixup}




\begin{icmlauthorlist}
\icmlauthor{Bingzhe Wu}{tx}
\icmlauthor{Zhipeng Liang}{hkust}
\icmlauthor{Yuxuan Han}{hkust}
\icmlauthor{Yatao Bian}{tx}
\icmlauthor{Peilin Zhao}{tx}
\icmlauthor{Junzhou Huang}{tx}
\end{icmlauthorlist}

\icmlaffiliation{tx}{Tencent AI Lab}
\icmlaffiliation{hkust}{Hong Kong University of Science and Technology}

\icmlcorrespondingauthor{Bingzhe Wu}{bingzhewu@tencent.com}
% \icmlcorrespondingauthor{Eee Pppp}{ep@eden.co.uk}


\vskip 0.3in
]



\printAffiliationsAndNotice{} % otherwise use the standard text.

\begin{abstract}
Recently, federated learning has emerged as a promising approach for training a global model using data from multiple organizations without leaking their raw data. 
Nevertheless, directly applying federated learning to real-world tasks faces two challenges: (1) heterogeneity in the data among different organizations; and (2) data noises inside individual organizations. 

In this paper, we propose a general framework to solve the above two challenges
simultaneously. 
Specifically, we propose using distributionally robust optimization to mitigate the negative effects caused by data heterogeneity paradigm to sample clients based on a learnable distribution at each iteration. 
Additionally, we observe that this optimization paradigm is easily affected by data noises inside local clients, which has a significant performance degradation in terms of global model prediction accuracy. 
To solve this problem, we propose to incorporate mixup
techniques into the local training process of federated learning. We further provide comprehensive 
theoretical analysis including robustness analysis, convergence analysis, and generalization ability. Furthermore, we conduct empirical studies across different drug discovery tasks, such as ADMET property prediction and drug-target affinity prediction.

\end{abstract}
\section{Introduction}
Recently, Federated learning (FL) is emerging as a promising way to mitigate the information silo problem, in which data are spread across different organizations and cannot be exchanged due to privacy concerns and regulations.  \cite{fedavg,YangLCT19}. There are plentiful attempts to apply FL to various real-world applications including recommender system \cite{ecai_chen,chen_tist}, medical image analysis \cite{sheller2020federated} and drug discovery \cite{fl-qsar,xiongfeddrug}. For example, Chen et al. \cite{ecai_chen} combine the FL method and secure multi-party computation for building recommender systems among different organizations. 
He et al. \cite{fedgnn} proposes a federated GNN learning framework for various molecular property prediction tasks.

Despite the above progress, directly applying FL methods to real-world scenarios still faces severe challenges.
In this paper, we focus on two
critical challenges: (1) \emph{Inter-client data heterogeneity} (2) \emph{Intra-client data noise}.
Specifically, inter-client data heterogeneity refers to data from different clients that always have varying
distributions. This is also called the Non-IID issue in the FL research community \cite{Scaffold,Fedprox}. For example, different pharmaceutical companies/labs may use different high-throughput screening environments to measure the protein-ligand binding affinity, resulting in different DTI assays \cite{bento2014chembl,mendez2019chembl}.
% \bz{add refs}. 
% can have xxxx(assay and scaffold). 
The  intra-client data noise refers to the noise created during the data collection process within each client individually.
A common noise source in drug discovery is the \emph{label noise} caused by measurement noise or the situation where very low/high-affinity values are often encoded as a boundary constant, see  \cite{mendez2019chembl} for a detailed description for the various noise in the largest bioassay deposition website.

% \bz{add refs}.  
% \yt

Numerous efforts have been made to solve the first issue \cite{mohri_agnostic_2019,reisizadeh_robust_2020,Fedprox,Scaffold,ditto,fed_non_iid_graph} from different aspects. %A typical research line is to employ an adaptive training paradigm (by adding a client-aware regularizer \cite{Scaffold} or using an adaptive local optimizer \cite{reddi2021adaptive}) to improve the FL performance in the Non-iid setting (i.e., inter-client data heterogeneity).
%For example, prior work \cite{reddi2021adaptive} proposes an adaptive optimizer for local model updating (e.g., Adam \cite{kingma2015adam} and AdaGrad \cite{duchi2011adaptive}) which allows to implicitly adjust local learning rate according to the loss function computed based on local data. 
%However, most  work in this line lacks theoretical guarantees and needs to tune large amounts of hyper-parameters of the adaptive training paradigm to ensure model performance for different datasets, which is very challenging in real-world applications. 
A recent trend tries an essential approach to overcome the client  heterogeneity by incorporating distributionally robust optimization (DRO) \cite{duchi2011adaptive}  into conventional FL paradigm \cite{mohri_agnostic_2019,deng_distributionally_2021}.  
Recently, 
DRFA \cite{deng_distributionally_2021}  proposes to optimize the distributionally robust empirical loss, which combines different local loss functions with learnable weights. This method theoretically ensures that the learned model can perform well over the worst-case combination of local data distributions. 

Previous work using DRO has shown  effectiveness on some academic benchmarks such as Fashion MNIST and UCI Adults \cite{deng_distributionally_2021,mohri_agnostic_2019}. 
However, when these methods are applied to real-world tasks such as drug discovery, we observe poor performance and training instability. Intuitively, prior work \cite{deng_distributionally_2021, mohri_agnostic_2019} aims to pessimistically optimize the worst-case combination of local distribution which is typically dominated by the local training distribution with more noises. Therefore they suffer serious performance drop issue as shown in Figure \ref{fig:intro_exp}. From this figure, the use of distributionally robust optimization cannot improve the worst-case local RMSE but even significantly reduces the performance.
According to further empirical studies (see details in Section \ref{sec:exp}), we identify that \emph{Intra-client data noise} inside the local client data is the main cause of the above issues. Similar issues have also been discussed in the centralized setting \cite{zhai_doro_2021}, which
% A prior work \cite{zhai_doro_2021} 
points out that  DRO is sensitive to outliers which further hinder model performance. In this paper, we focus on mitigating the effects caused by intra-client data noise in the FL setting.

In order to address the above two critical challenges, we present a general framework named \rfdd\ for building robust neural network models in the FL setting for a wide range of tasks.
To overcome the inter-client data heterogeneity issue, we propose to employ DRO techniques to optimize the global model.
As discussed above, such a DRO method can be easily affected by intra-client data noise \cite{zhai_doro_2021}.
As suggested by previous work \cite{zhai_doro_2021}, a natural method to reduce the noise effect in the centralized setting is to drop samples with large losses during training. 
However, this approach may result in samples that are very valuable to overall training being discarded, which then leads to sub-optimal results.
Hence, we propose to integrate the mixup strategy into the local training process and find that it can significantly improve the empirical performance.

Furthermore, this paper offers an interesting theoretical finding that mixup can provide heterogeneity-dependent generalization bounds, as shown in Theorem \ref{thm: Generalization guarantee}.
In contrast to most FL-related work which mainly focuses on convergence analysis, our theoretical analysis examines more aspects of the proposed FL approach. 
Specifically, our analysis consists of three parts: 
(1) We provide a concrete example to demonstrate the effectiveness of the mixup strategy for mitigating noise effects on conventional DRO FL methods. 
(2) We analyze the generalization bounds of the learned model,  which is the primary theoretical contribution of this paper. 
(3) We provide rigorous convergence analysis of the training process. Based on theoretical analysis, we conduct extensive empirical studies covering a wide range of drug discovery tasks including molecular property prediction and drug-target interaction prediction.  Extensive experimental results show the effectiveness of DRFLM in reducing the effects of client distribution shift and local data noise.

%(Under construction, illustrating different kinds of noises in drug AI datasets): 
%There are several categories of noise from the area of AI-aided drug discovery, e.g, for CHEMBL \citep{mendez2019chembl}:  measurement noise, narrow measurement range,   low/high values are often only encoded as a boundary constant et al.   

% Our major contributions are summarized as follows:
% \begin{itemize}
%     \item We propose a general framework called \rfdd  to  address the two challenges (inter-client data heterogeneity and intra-client data noise) simultaneously when  applying FL methods;
%     % consider two types of noises in FL setting
    
%     \item  We present rigorous  theoretical analysis on the generalization bounds of the model  and convergence of the training process; 
%     \item We empirically demonstrate the superior performance of \rfdd on a variety of real-world noisy drug discovery tasks.   
% \end{itemize}

\begin{figure}
    \centering
    \includegraphics[width=0.3\textwidth]{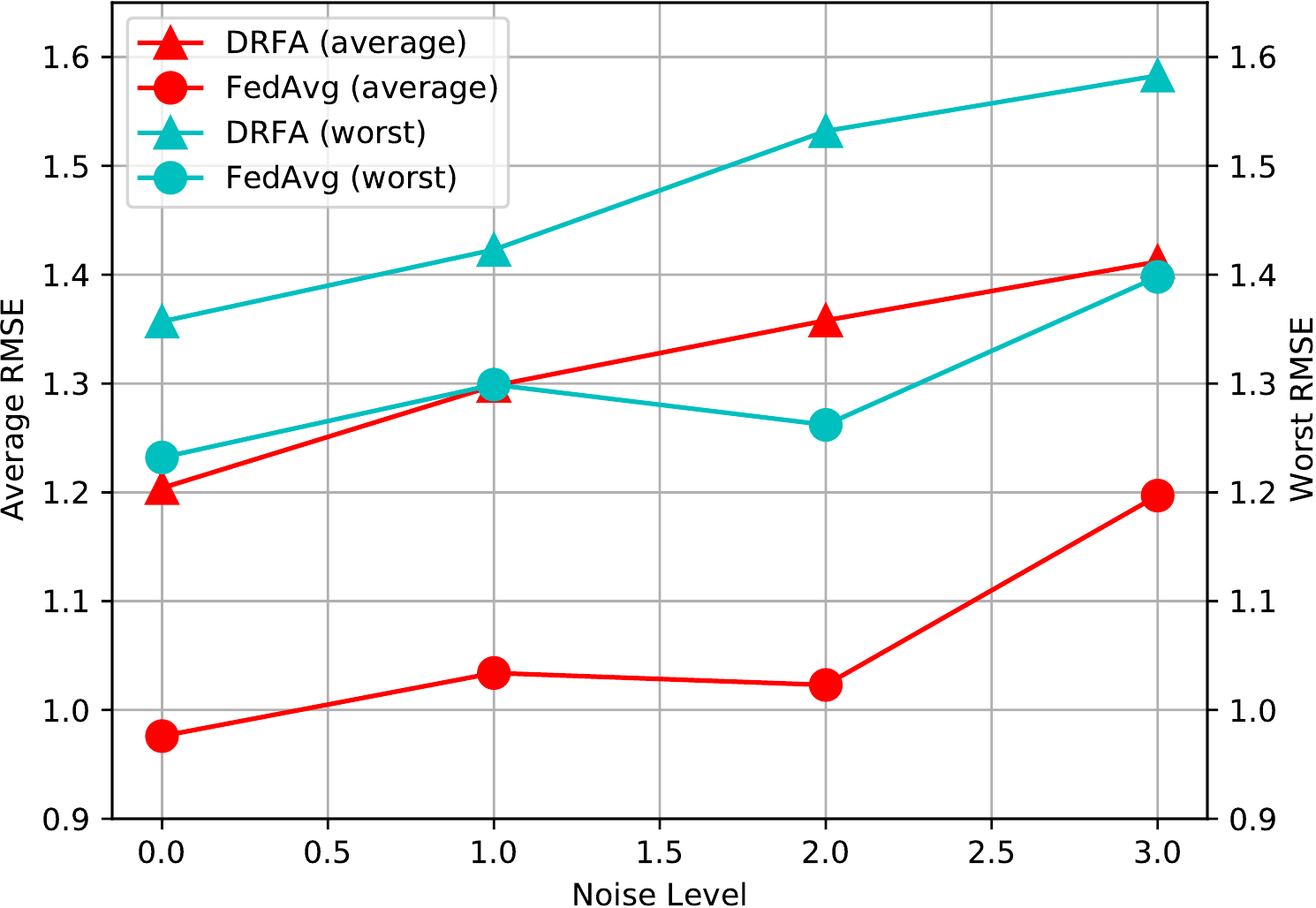}
    \caption{Average and worst-case client RMSE of  FedAvg \cite{fedavg} and DRFA \cite{deng_distributionally_2021} on a molecular property prediction task. }
    \vspace{-0.5cm}
    \label{fig:intro_exp}
\end{figure}

\section{Framework of \rfdd}
\subsection{Notation}
We start by fixing some notations that will be used throughout this paper. 
% 	% basic
% 	For a positive integer $n$, $[n]$ denotes the set $\{1,\cdots, n\}$. $\lvert A \rvert$ denotes the cardinality of the set $A$.
% 	%	The inner product of two vectors is written as $\langle x, y\rangle = \sum_{i=1}^d x_iy_i$.
% 	$\lVert \cdot \rVert_2$ is Euclidean norm.
% 	$W(i,j)$ denotes the element in the $i$-th row and $j$-th column of matrix $W$. 
% 	We denote $W>0$ if the matrix $W$ is symmetric and positive definite. We denote $I_d$ as the $d$-dimensional identity matrix.
% 	Let $\otimes$ denote the Kronecker product.
% 	%
% 	Let $B_r^d$ denote the $d$-dimensional ball with radius $r$ and $S^{d-1}_r$ denotes the $(d-1)$-dimensional sphere for the ball. 
% 	%
% 	Given a set $A$, Unif$(A)$ denote the uniform distribution over $A$.
% 	%
% 	For a tuple $(Z_{i,j})_{i\le N, j\le M}$ and $1\le k_1< k_2\le M$, we denote $Z_{i, k_1:k_2} = (Z_{i,k_1}, \cdots, Z_{i,k_2})$. 
% 	% Asymptotic
% 	We adopt the standard asymptotic notations: for two non-negative sequences $\{a_n\}$ and $\{b_n\}$, $\{a_n\}=O(\{b_n\})$ iff $\lim \sup_{n\rightarrow \infty}a_n/b_n<\infty$, $a_n = \Omega (b_n)$ iff $b_n = O(a_n)$, $a_n = \Theta(b_n)$ iff $a_n = O(b_n)$ and $b_n = O(a_n)$. We also write $\tilde{O}(\cdot)$, $\tilde{\Omega}(\cdot)$ and $\tilde{\Theta}(\cdot)$ to denote the respective meanings within multiplicative logarithmic factors in $n$.ß
we denote $[N] = \{1,2,\cdots,N\}$.
We consider federated learning setting where $N$ clients jointly train a central model with the i-th client's local training datasets of sample size $N_i$. The local dataset is denoted as $\mathcal{D}_i=\{\boldsymbol{z}_{i,j}\}_{j=1}^{N_i}$ and each sample $\boldsymbol{z}_{i,j}=(\boldsymbol{x}_{i,j}, y_{i,j})$ is i.i.d. drawn from an unknown local joint distribution $\mathcal{P}_{x,y}^{(i)}$.
We denote $Z^{\dagger}$ for any matrix $Z$ as the pseudo inverse of it.
% For a matrix $A$, we denote $s_{\min}(A)$ the minimal eigenvalue of $A.$

\subsection{Optimization Objective}
A common used approach in Federated learning (FL) to collaboratively train a central model $\boldsymbol{w}$ is to optimize the following objective:
\begin{equation}
\label{eq:FL}
    \min_{\boldsymbol{w}\in \mathcal{W}} F(\boldsymbol{w}) \coloneqq \sum_{i=1}^N \frac{N_i}{N} f_i(\boldsymbol{w})
\end{equation}
where $f_i(\boldsymbol{w})= \frac{1}{N_i} \sum_{\boldsymbol{z}_{i,j} \in \mathcal{D}_i} f(\boldsymbol{w}, \boldsymbol{z}_{i,j})$ is the local empirical risk for client $i$ with the global model. %$\boldsymbol{w}$, $N$ is the number of clients and $N_i$ is the sample size of the local dataset for client $i$.
The optimization of Equation \ref{eq:FL} suffers from non-iid data distribution, i.e., the heterogeneity across different clients' data collection \cite{deng_distributionally_2021}.
To provide an intuition, note the target objective in Equation \ref{eq:FL} aggregates client-specific loss functions with weights proportional to their local dataset sizes following the standard empirical risk minimization framework.
Imbalance sample sizes of different clients or diversity among local data distributions may inevitably plague the generalization ability of the central model obtained by solving Equation~\ref{eq:FL}. 
% A few notable studies attempt to address the this issue by personalizing the global model to local distributions \cite{deng2020adaptive, mansour2020three}.
To provide the performance guarantee over the worst-case local distribution,  Deng et al. \cite{deng_distributionally_2021} propose a framework named DRFA to minimize the worst-case combination loss of local empirical distributions:
\vspace{-1em}
\begin{equation}\label{eq: DRFA loss}
    \min_{\boldsymbol{w}\in \mathcal{W}} \max_{\boldsymbol{\lambda}\in \Lambda} F(\boldsymbol{w}, \boldsymbol{\lambda}) = \sum_{i=1}^N \lambda_i f_i(\boldsymbol{w})
\end{equation}
\vspace{-1.5em}
where $\boldsymbol{\lambda}\in \Lambda \coloneqq \{\boldsymbol{\lambda}\in \mathbb{R}^N_+, \sum_{i=1}^N \lambda_i = 1\}$ is the learnable sampling weights for each local client and $f_i(\cdot)$ is the local empirical distribution.

However, Equation~\ref{eq: DRFA loss} tends to pessimistically optimize the worst-case local distribution which can be easily affected by inter-client data noise.
As shown in Figure \ref{fig:intro_exp}, when we apply these methods to real-world drug discovery datasets,  the use of distributionally robust optimization cannot improve the worst-case local RMSE but even significantly reduces the performance.

%However, the issue that distributional robust optimization is prone to suffer low generalization and poor instability during training with outliers has been realized by several papers \cite{hu2018does, zhu2019generalized, zhai2021doro}.
%With incredibly large number of clients and heterogeneous local data collection procedure in FL, central model obtained by solving~\ref{eq: DRFA loss} can unlikely to escape from the same trap in practice. 
% While such method can alleviate the issue of data heterogeneity across clients, when local client's data distribution is noisy or distorted, it can perform badly \cite{zhai2021doro}. 
To mitigate the noise effect on the optimization procedure, instead of directly optimizing the worst-case local empirical distribution, 
%a naive way is to drop samples with large training loss,  
we propose to incorporate the popular data augmentation technique, mixup, into the local training procedure as follows:  

% We also introduce a regularizer on $\boldsymbol{\lambda}$ and propose the following optimization for global model
\vspace{-1.5em}
\begin{equation}
\label{eq: rfdd-optimization}
    \min_{\boldsymbol{w}\in W} \max_{\boldsymbol{\lambda}\in \Lambda} F(\boldsymbol{w}, \boldsymbol{\lambda}) \coloneqq \sum_{i=1}^N \lambda_i \tilde{f}_i(\boldsymbol{w}).
\end{equation}
Here we optimize the worst-case local empirical loss evaluated on the mixup distribution $\tilde{f}_i$.
To be specific, for the $i$-th client's local dataset, we denote $\tilde{\boldsymbol{x}}^{(i)}_{j,k}(\gamma) = \gamma g(\boldsymbol{x}_{i,j}) + (1-\gamma) g(\boldsymbol{x}_{i, k})$ and $\tilde{y}^{(i)}_{j,k}(\gamma) = \gamma y_{i,j} + (1-\gamma) y_{i, k}$ with $\gamma \in [0,1]$ be the linear interpolation between two samples in the local dataset $\mathcal{D}_i$ and $g:\mathbb{R}^d\rightarrow \mathbb{R}^d$ is a mapping to improve the flexibility of our framework to handle different kinds of data. 
We choose $g(\cdot)$ as an identity operator for the tableau or image data.
For graph prediction tasks including drug discovery, we choose $g(\cdot)$ as a graph neural network to map a graph $(V, E)$ to a global graph embedding in $\R^d$.
Let $\gamma$ follows the distribution as $\operatorname{Beta}(\alpha, \beta)$, $\tilde{\boldsymbol{z}}^{(i)}_{j,k}(\gamma)=(\tilde{\boldsymbol{x}}_{j,k}^{(i)}(\gamma), \tilde{y}_{j,k}^{(i)}(\gamma))$ and $\tilde{\mathcal{D}}_i(\gamma) = \{\tilde{\boldsymbol{z}}^{(i)}_{j,k}(\gamma), j,k\in [N_i]\}$ be the local mixup dataset by linear interpolating between every two samples.
After the construction of the mix-up empirical distribution, we evaluate each local loss function on the local mixup dataset, i.e.,  $\tilde{f}_i(\boldsymbol{w})=\frac{1}{N_i^2}\sum_{\tilde{\boldsymbol{z}}_{j,k}^{(i)}\in \tilde{\mathcal{D}}_i(\gamma)}f_i(\boldsymbol{w}, \tilde{\boldsymbol{z}}_{j,k}^{(i)})$ and then combine them with learnable sampling weights $\lambda_i$ as shown in Equation \ref{eq: rfdd-optimization}.
% Moreover, we add a regularizer $R(\cdot)$ to allow a wide range of choices such as KL-divergence, or $\ell_p$ norm to leverage the domain prior information.

% The strategy of mix-up follows Manifolds Mixup \cite{verma2019manifold} which mixups the representation rather than the original input. 
\subsection{Optimization Procedure}
We adopt a similar approach in \cite{deng_distributionally_2021} to solve the min-max optimization problem in Equation~\ref{eq: rfdd-optimization}. The core idea is to alternatively update $\boldsymbol{w}$ and $\boldsymbol{\lambda}$ combining with sample techniques to reduce communication cost.
We briefly introduce the main steps in Algorithm~\ref{algorithm:1} (more details can be found in the prior work \cite{deng_distributionally_2021}).
At the beginning of the $s$-th communication stage, the optimization procedure follows as:\vspace{-1em}
\begin{enumerate}
\item Sampling: the server independently selects two subsets of $[N]$, $C_s$ and $U_s$, both with size of a pre-define parameter $m$ and also samples a time index $t^{'}_s\sim U[0,\tau]$ with a pre-define parameter $\tau$ for updating $\boldsymbol{\lambda}$.
    The selection of $C_s$ follows the probability $\boldsymbol{\lambda}_t\in \R^N$ with replacement. 
    \item Broadcasting: the server broadcasts the central model to the selected clients.
    Then the client constructs a mixup dataset as discussed above and runs stochastic gradient descent (SGD) algorithm on the mixup dataset to update their local models.
    After that clients return their local parameter updates after time $t^{'}_s$ (used for updating $\boldsymbol{\lambda}$) and the last time $\tau$ (used for updating the global model).
    \item Updating $\bar{\boldsymbol{w}}_{s+1}$: the server takes an average over the received latest model parameters. To be specific, 
    $$
    \bar{\boldsymbol{w}}_{s+1} = \frac{1}{m}\sum_{j\in C_s} \w_{s,\tau}^{(j)}
    $$ 
    \item Updating $\boldsymbol{\lambda}_{s+1}$: the server takes an average $\tilde{\w}_{s+1} = \frac{1}{m} \sum_{j\in C_s}\w_{s, t^{'}_{s}}^{(j)}$ and then broadcasts it to all clients in $U_s$ and receives the stochastic gradients evaluated at this point to construct an auxiliary vector $\boldsymbol{v}_s$. Then the server updates $\boldsymbol{\lambda}$ as following with some pre-defined learning rate $\eta:$ 
    $$
    \boldsymbol{\lambda}_{s+1} = \prod_{\boldsymbol{\lambda}\in \Lambda}(\boldsymbol{\lambda}^{(s)}+\eta \tau \boldsymbol{v}_s).
    $$
\end{enumerate}
\vspace{-1em}

\begin{algorithm*}[tb]
   \caption{Dristributionally Robust  Federated Learning with Local Mixup  (DRFLM)}
   \label{algorithm:1}
\begin{multicols}{2}
\begin{algorithmic}
   \STATE \hskip-0.75em \textit{parameters:}
   \STATE synchronization gap $\tau$, total communication round $T$, $S = T/\tau$, sampling size $m$, mix-up ratio $\gamma$, initial model $\bar{\boldsymbol{w}}_0$ and initial weights $\boldsymbol{\lambda}_{0}$

   \medskip
   
   \STATE \hskip-0.75em {\bfseries function} \textsc{ClientUpdate}($i$, $\bar{\boldsymbol{w}}_s,t_s^{'}$)
        \STATE Initialize local weights $\boldsymbol{w}_{s,0}^i = \bar{\boldsymbol{w}}_s$
        \STATE Construct $\mathcal{\tilde{D}}_i(\gamma) = \{\gamma g(\boldsymbol{z}_1)+(1-\gamma)g(\boldsymbol{z}_2)\lvert \boldsymbol{z}_1, \boldsymbol{z}_2\in \mathcal{D}_i\}$
        \FOR{$t$ = $1$ to $\tau$}
            \STATE Uniformly sample $\tilde{\boldsymbol{z}}_{t}^{(i)}$ from $\mathcal{\tilde{D}}_i(\gamma)$
            \STATE  $\boldsymbol{w}_{s,t}^{(i)} = \prod_{\mathcal{W}}\left ( \boldsymbol{w}_{s,t-1}^{(i)} - \eta\nabla f(\boldsymbol{w}_{s,t-1}^{(i)};\tilde{\boldsymbol{z}}_{t}^{(i)})\right)$
            %\STATE Update $(\theta_i, w_i)$ using back-propagation
        \ENDFOR
    \STATE return ($\w_{s, \tau}^{(j)}, \w_{s, t^{'}_s}^{(j)}$)
    \medskip
 
   \STATE \hskip-0.75em {\bfseries Server executes:}
   \STATE initialize $\bar{\boldsymbol{w}}_0$
   \FOR{$s=1$ {\bfseries to} $S$}
   \STATE Sample a set $C_s\subset [N]$ with size of m and probability $\boldsymbol{\lambda}_s$ with replacement
   \STATE Uniformly sample $t^{'}_{s}$ from $[\tau]$ uniformly 
   \STATE Uniformly sample a set $U_s\subset [N]$ with size of m
%   \STATE Broadcasts $\bar{w}_s$ and $t^{'}_s$ to all clients $i\in C_s$
   \FOR{client $j\in C_s$ {\bfseries in parallel}}
   \STATE $(\w_{s, \tau}^{(j)}, \w_{s, t^{'}_s}^{(j)}) = \textsc{ClientUpdate}(j, \bar{\boldsymbol{w}}_s, t_s^{'})$
   \ENDFOR
   \STATE $\bar{\w}_{s+1} = \frac{1}{m} \sum_{j\in C_s} \w_{s,\tau}^{(j)}$\\
   \STATE $\tilde{\w}_{s+1} = \frac{1}{m} \sum_{j\in C_s} \w_{s, t^{'}_s}^{(j)}$\\
   \STATE Broadcasts $\tilde{\w}_{s+1}$ to all clients $i\in U_s$ and receive $f_i(\tilde{\w}_{s+1}; \tilde{\boldsymbol{z}}_{t^{'}_s}^{(i)})$ from local client
   \STATE Construct $\boldsymbol{v}\in \mathbb{R}^N$ with $v_i=\frac{N}{m}f_i(\tilde{\w}_{s+1}; \tilde{\boldsymbol{z}}_{t^{'}_s}^{(i)})$ for each indice $i\in U_s$ and $v_i=0$ for $i \notin U_s$
   \STATE Updates $\boldsymbol{\lambda}_{s+1} = \prod_{ \Lambda}(\boldsymbol{\lambda}_{s}+\eta \tau \boldsymbol{v}_s)$
   \ENDFOR

   \STATE Output $\bar{\boldsymbol{w}}_S,\boldsymbol{\lambda}_S$
   
\end{algorithmic}
\end{multicols}
\end{algorithm*}

\section{Theoretical Analysis}\label{sec: theory section}
In this part, we  provide 
different theoretical views to understand our approach in Algorithm \ref{algorithm:1}.
\subsection{Theoretical Motivation}
We first show a concrete example to demonstrate the effectiveness of our method in mitigating the local noise effect. Specifically, we construct an example as follows: 
%The following example illustrates that DRFA can be more vulnerable than naive FedAvg to local client noise: 
\begin{example}\label{example: exmp1}
We consider FL setting with two clients. These two clients jointly train  a linear classifier: $\phi(\x;\w) =  \text{sgn}(\w^T\x)$ over the global dataset $(\x,y)$ generated by the joint distribution $P(\x,y)$ as following: \begin{align*}
    &\x\sim P_X:=\text{Uniform}\bigg\{[-2,-1]^{d}\bigcup \big([1,2]\times [-2,-1]^{d-1}\big)  \bigg\}  \\ 
    &y=\left\{\begin{matrix} 1,& x_1 >0\\
    -1,&x_1\leq 0\end{matrix}\right.. 
\end{align*} Given the above global joint distribution, the local distribution of each client is defined as following:
\vspace{-1em}
\begin{itemize} 
    \item The "clean" dataset $D_1=\{(\x_{1,i},y_{1,i})\}_{i=1}^{N_1}\overset{i.i.d.}{\sim} P(\x,y)$ of the first client is noiseless.
    \item The "noisy" dataset  $\tilde{D}_2 =\{(\x_{2,i},\tilde{y}_{2,i})\}_{i=1}^{N_2} \overset{i.i.d.}{\sim}\tilde{P}(\x,\tilde{y})$ ($\tilde{y}$ is a noisy version of $y$ defined as follow): \begin{align*}
        &{\tilde{y}}_{2,i} = \left\{\begin{matrix} 
            y_{2,i}, & x_1\notin I\\
            \text{Flipping}(y_{2,i};p_1), & x_1\in I
        \end{matrix}\right.
        \end{align*}
    with $I = [-1-p_2,-1],1/2<p_1<1, 0<p_2<1$ and $\text{Flipping}(y;p_1)$ the random function that flips $y$ to $-y$ with probability $p_1.$ We construct $\tilde{D}_2$ as a noisy version of $D_2$ to simulate the local noise effect on the  training procedure.
    \end{itemize} 
    \vspace{-1em}
 Suppose $N_1 = N_2 = \tilde{N}$, then given the 0-1 loss $f(\w;\x,y): = \bm1\{y \neq \phi(\x;\w) \} $ and the learned models $\w^\text{Avg}_{\tilde{N}},\w^\text{DRFA}_{\tilde{N}},\w^\text{\rfdd}_{\tilde{N}}$ obtained by minimizing the empirical FedAvg loss (Equation \ref{eq:FL}), DRFA loss (Equation \ref{eq: DRFA loss}) and \rfdd loss (Equation \ref{eq: rfdd-optimization}) associated with $\ell$. We then have~\footnote{here the mix-up parameter is specified as $\alpha = \beta $}
    \begin{align*}
       &\lim_{\tilde{N}\to\infty} \E_{P(\x,y)}[f(\w^\text{DRFA}_{\tilde{N}},\x,y)] \\
       >&\lim_{\tilde{N}\to\infty} \E_{P(\x,y)}[ f(\w^\text{Avg}_{\tilde{N}};\x,y)]\\
       =&\lim_{\tilde{N}\to\infty}  \E_{P(\x,y)}[f(\w^\text{\rfdd}_{\tilde{N}};\x,y)] = 0.
    \end{align*}
%See figure~\ref{fig:visulization} as the visulization.
\end{example}
% \begin{figure}[htbp]
%     \centering
%     \includegraphics[width=0.4\textwidth]{Example.pdf}
%     \caption{Visulization of the example}
%     \label{fig:visulization}
% \end{figure}
%\yx{CANDO: Compare the Empirical Loss Landscape of DRFA loss before/after mixup.}
As shown by the above inequality, under certain conditions, with the increase of training sample number,  losses of FedAvg and our method can gradually decrease to zero while the loss of DRFA cannot converge to zero even with infinity samples. 
We have two insights from the above example: (1)
Directly applying DRO method to FL settings can be easily affected by the local noise. (2) The use of local data mixup has the potential to mitigate the local noise.

% \yx{I have also prepared an example on DRFA = RFDD $>$ FEDAVG, but I wonder whether this example can illustrate the point that RFDD can maintain cross-client robustness. Maybe the theorem I claimed in  next subsection can illsturate this point better. }

\subsection{Generalization Guarantee of \rfdd}
% \yx{Contributions of new result: \begin{enumerate}
% 	\item Characterizing the role of mix-up regularization when taking maximal over $\lambda$, which is not a straightforward extension of classical setting, and can be further explored in more general settings.
% 	\item Relating the mix-up regularization effect to a new heterogeneous-dependent term $H_j$.  
% \end{enumerate} }
In previous subsection, we have demonstrated a motivated example to show the effectiveness of local data mixup. 
In this part, we focus on analyzing the
generalization bound of the worst-case client loss  in the setting where a generalized linear model is built using our approach.
Here, we take generalized linear model as an example and our analysis can be also extended to more complex hypothesis sets such as two-layer neural networks (see more general results in the Appendix~\ref{appendix:non-linear}): 
\begin{equation}\label{eq: GLM density}
P({y}\lvert \x;\w) \propto \exp({y}\w^T\x- \mu (\w^T\x) )
\end{equation}
with $\x,\w\in \mathbb{R}^d$ and $\mu(\cdot)$ is some twice differentiable function.  
The corresponding negative log likelihood $f$ is: \begin{align}
    f(\w;\x, y) = \mu (\w^T\x)-y\w^T\x. 
\end{align}
Typical losses such as cross entropy can be obtained by setting different $\mu$.
Differentiate $\int P(y\lvert \x;\w)dy = 1$, we get $\E[y\lvert \x,\w] = \mu'(\w^T \x)$, thus Equation \eqref{eq: GLM density} can be write equivalently as:
\begin{equation}\label{eq: GLM equ form}
	y  = \mu'(\w^T \x)+ \varepsilon,
\end{equation}
with $\varepsilon$ as a random variable with zero mean. We further assume that for $\lvert z\rvert \leq 1 $ , there exists some $K$ such that  \begin{equation}\label{assumption: on GLM K}K^{-1} \leq \mu''(z)\leq K.
\end{equation} 
%\yx{This condition implies the condition needed in first inequality of Theorem 1, we will provide examples on $K$ in Linear case and Logistic case later} . 
In particular, the linear model and the Logistic model satisfies our conditions. Here, we take Linear model for regression as an example: When we set $\mu (z) = \dfrac{1}{2}z^2,$  Equation \eqref{eq: GLM equ form} turns to be $y = \w^T \x + \varepsilon$, and $f(\w;\x,y) = \dfrac{1}{2}(y-\w^T \x)^2+ \dfrac{1}{2}y^2 $
corresponds to the squared loss. By $\mu ''(z) = 1$, we have $K = 1$ satisfies Equation \eqref{assumption: on GLM K}.
%(The example of Logistic regression is shown in the appendix).
% \begin{example}[Logistic Model]
% When $y$ is supported in $\{0,1\}$ and  $\mu(z) = \log(1 + e^z)$, we have \eqref{eq: GLM equ form} turns to \begin{align*} &P(y = 0;\w,x) = \dfrac{1}{1+\exp(\w^T x)},\\ 
% &P(y = 1;\w,x) =  \dfrac{\exp(\w^T x)}{1+\exp(\w^T x)}  \end{align*}
% and \begin{align*}
% 	f(\w;x,y) = -y \log P(y;\w,x)- (1-y) \log P(1 - y;\w,x)
% \end{align*}
% corresponds to the cross-entropy loss. By $\mu ''(z) = e^z/(1+e^z)^2$, we have $K =  e^2/(1+e^2)^2$ satisfies the assumption\eqref{assumption: on GLM K}.
% \end{example}

At a high level, our theory aims to characterize
the difference between the population and the empirical (worst-case) losses (i.e., generalization bound of the worst-case loss). The derivation of the generalization bound follows two steps:  
\vspace{-1em}
\begin{itemize}
    \item Firstly, we would show that as long as the estimator $\w$ lies in some constrained hypothesis class $\mathcal{W}$, then the generalization ability of $\w$ over the worst-case loss is theoretically guaranteed.
    \item Then, we argue that the mix-up effect is nearly a penalty that forces the trained estimator fall into the above hypothesis class $\mathcal{W}$.
\end{itemize}
\vspace{-1em}
\noindent\textbf{Step1:}  Without loss of generality, we denote the $i$-th local dataset as $\Dset_i =\{\z_{i,j} = (\x_{i,j},y_{i,j})\}_{j=1}^{N_i}$ and assumes $\lVert \x_{i,j}\rVert_2, \lVert \w^*\rVert_2\leq 1.$  We  denote \begin{align*} &\hat{\Sigma}_i: = \dfrac{1}{N_i}\sum_{j=1}^{N_i}\x_{i,j}\x_{i,j}^T,\quad \zeta_i(\w): = \dfrac{1}{N_i}\sum_{j=1}^{N_i}\mu{''}(\w^T \x_{i,j}),\\
&\Sigma_\Lambda : = \{\sum_{j=1}^N \lambda_i \hat{\Sigma}_j :\lambda \in \Lambda  \}. 
\end{align*}
Here $\Sigma_i$ and $\hat{\Sigma}_i$ denote the population and empirical covariance matrix of the $i$-th client, respectively.
We further introduce the following hypothesis class, which can be seen as a federated-setting version of the hypothesis set considered in \cite{zhang_how_2021}:
\begin{align*}
        \mathcal{W}_r: = \{ \x\to\w^T\x: \min_{\Sigma \in \Sigma_\Lambda}  \w^T\Sigma \w \leq r \}.
\end{align*}
In particular, the above hypothesis class involves a quadratic form $\w^T \bar{\Sigma}(\w) \w$, with $$\bar\Sigma(\w): =  \text{argmin}_{\Sigma_{\Lambda}}\{ \sum_{i=1}^N \lambda_i \w^T \Sigma_i\w \},
$$
the weighted average of the data covariance matrices across different clients. %In particular, we have $\text{rank}( \bar\Sigma_\lambda)\leq \text{rank}(\bar\Sigma): = \text{rank}(\sum_{i=1}^N  \Sigma_i) $. 

In fact, we will show later that $\mathcal{W}_r$ reflect the effect of mix-up comparing with DRFA and FedAvg: Without the mix-up effect, the only information we can infer about the empirical minimizer of DRFA and FedAvg is that $ \lVert \w\rVert_2\leq 1$, while the \rfdd\ loss will push its minimizer into $\mathcal{W}_r$ due to the mix-up effect, as we will show in Step2. This will provide a more tight generalization bound in most cases.
% so in Theorem~\ref{thm: Generalization guarantee} , the generalization bound will involve $\sqrt{\dfrac{d}{n}}$ term instead of $\sqrt{\text{rank}(\tilde{\Sigma})/n}$ term. 

Now we can show the following generalization bound over the hypothesis class defined above: 
\begin{theorem}[Generalization Bound of $\mathcal{W}_r$]\label{thm: Generalization guarantee} Suppose $\mu(\cdot)$ is $L$-Lipshitz continuous.
% and \begin{equation}\label{eq: assumption rho}
%         \E_{\x\sim P_j}[\mu''(\x^T\w)]^2 \geq \rho\cdot\min\{1,\E_{\x\sim P_j}[(\w^T\x)^2] \},\quad j \in[N].
% \end{equation} \yx{This condition is covered by the assumption when introducing GLM, so will be deleted.} 
Then there exists constants $B>0$ such that  for  $\w\in\mathcal{W}_r$, the following inequality holds  with probability $1-\delta$:
\begin{align*} 
&\max_{\lambda\in\Lambda}    \sum_{j=1}^N\lambda_j\E_{P_j(\x,y)}[f(\w;\x,y)]\\
&\leq  \max_{\lambda\in\Lambda} \sum_{j=1}^N \dfrac{\lambda_j}{N_j}\sum_{i=1}^{N_j} f(\w;x_j^i,y_j^i)\\
& + \max_{\lambda\in\Lambda}\sum_{j=1}^N \lambda_j\sqrt{\dfrac{1}{N_j} }\big(\sqrt{\log\dfrac{N}{\delta} }+\sum_{j=1}^N rL\cdot \textcolor{red}{\sqrt{H_j/N_j}} \big),
\end{align*}
with $H_j: = \E_{P_j}[ \max_{\Sigma \in \Sigma_\Lambda} \text{tr}\big(\Sigma^{\dagger}\hat{\Sigma}_j\big)] 
$ .
\end{theorem}
 $H_j$ can partially reflects the client heterogeneity from the global distribution. In the IID case, the red part in above inequality turns to be $O(\sqrt{\text{rank}(\Sigma)/N_j})$. 
 In contrast, in both FedAvg and DRFA, this term will be $\Omega(\sqrt{d/N_j})$. 
 Thus our bound can be more tight when the intrinsic dimension of data is small (i.e., $\text{rank}(\Sigma)\ll d$).

%In comparison, if we only consider the hypothesis class of  $ \lVert \w\rVert_2\leq 1$ induced by  DRFA and FedAvg, the "$H_j$" term (marked with red color) in above inequality turns to be $\Omega(\sqrt{d/n})$. Thus our bound can be more tight when the intrinsic dimension of data is small as pointed out in the prior work \cite{zhang_how_2021}.
%\yx{To find some intuition on $H_j$, notice that as $n_j \to\infty$ we have $H_j \to \max_{\Sigma \in \Sigma_\Lambda} \text{tr}\big(\Sigma^{\dagger}{\Sigma}_j\big)$, i.e. the term $H_j$ reflects the heterogeneity: In best case it scales at $\text{rank}(\Sigma_j),$ in worst case it scales at $d$, which is same as no-regularization case. And we can draw a generalization bound on this quantity now.}

%$H_j$ reflects the distribution deviation 
%intrinsic low-rank structure on the mixed distribution. 

%may grows to the worst case $d$ when the  data distribution is highly heterogeneous across different clients, the term can be reasonable smaller than $d$  when the data distribution is homogeneous across different clients because the low-rank structure can be preserved after mixup.

% \begin{remark}
% In fact, the assumption \eqref{eq: assumption rho} is satisfied as only as the underlying parameter $\w^*$ is bounded.
% \end{remark}    

\noindent\textbf{Step2: }Now we 
prove that the learned model obtained by our approach can fall into $\mathcal{W}_r(\lambda)$ due to the mixup effect. Employing similar techniques as in \cite{zhang_how_2021}, we can show that the object function of \rfdd\ turns approximately to \begin{align}\label{eq: FedAvg with localr}
    \min_{\w \in W}\max_{\lambda\in \Lambda} \sum_{i=1}^N \lambda_i ({f}_i(\w)+   \underbrace{\dfrac{c}{2} \zeta_i(\w) \w^T\hat\Sigma_i \w   }_{R_i(\w)} ).
\end{align}
Then for any possible $\blambda \in \Lambda$ the mixed penalty term $ R(\w):=  \sum_{i=1}^N\lambda_i R_i(\w) $ forces the estimator to enter the set $\mathcal{W}_r$ for some $r$.
The formal version of the proof is refereed to Appendix~\ref{thm: Generalization guarantee}.
% For $\mathcal{W}_\gamma(\Lambda)$, we have \begin{align*}
%     \text{Rad}_n(\mathcal{W}_\gamma(\Lambda) ) \leq \dfrac{\max\{(\gamma/\rho)^{1/2},(\gamma/\rho)^{1/4} \}}{c(\Lambda)\sqrt{n}}
% \end{align*}

% \begin{lemma}
% Suppose client $j$ has the centralized dataset $\{(x_i,y_i)\}_{i=1}^n$ such that $1/n_j\sum_{i=1}^n x_i = 0$ and denote $\hat{\Sigma}_j = \dfrac{1}{n_j}\sum_{i=1}^n x_ix_i^T.$ Then if $\mu(\cdot)$ is twice differentiable,  the second order approximation of $\tilde{f}_j(\w) $ is given by \begin{align}
%   &\hat{f}_j(\w;D_j): =\\
%   &\dfrac{1}{n_j}\sum_{i=1}^{n_j}{f}(\w;x_i,y_i) +c(\gamma) \dfrac{1}{2n_j} \big[\sum_{i=1}^{n_j} \mu''(\w^Tx_i) \big]\w^T\hat{\Sigma}_X\w. 
% \end{align}
% with $c(\gamma) = \E_{D_\gamma}[\dfrac{(1-\gamma)^2}{\gamma^2} ], D_\gamma = \dfrac{\alpha}{\alpha+\beta}\text{Beta}(\alpha+1,\beta) +\dfrac{\alpha}{\alpha+\beta} \text{Beta}(\beta+1,\alpha). $
% \end{lemma}

% In particular, the mix-up effect converts to the penalty term $$R_j(\w): =c_j(\gamma) \dfrac{1}{2n_j} \big[\sum_{i=1}^{n_j} \mu''(\w^Tx_i) \big]\w^T\hat{\Sigma}_j\w. $$

\subsection{Convergence analysis}
In fact, we can establish similar convergence result for our algorithm \ref{algorithm:1} under non-convex loss functions by following the proof of Theorem~4 in \cite{deng_distributionally_2021} and show that our algorithm can efficiently converge to a stationary point of empirical loss.
Following the standard analysis of the nonconvex minimax optimization, we define $\blambda^{*}(\w)=\operatorname{argmax}_{\blambda\in \boldsymbol{\Lambda}}F(\w, \blambda)$ and $\Phi(\w)=F(\w, \blambda^{*}(\w))$.
To derive meaningful results, the convergence measure concerns the so-called Moreau envelope of $\Phi$ \cite{moreau1965proximite, davis2019stochastic}.
\begin{definition}[Moreau Envelope]
A function $\Phi_p(\x)$ is a $p$-Moreau envelope of a function $\Phi$ if $\Phi_p(x) = \min_{\w\in \mathcal{W}}\{\Phi(\w)+\frac{1}{2p}\lVert \w -\x \rVert^2\}$.
\end{definition}
We use $\frac{1}{2L}$-Moreau envelope of $\Phi$ and the proof is similar to Theorem 2~\cite{deng_distributionally_2021} except we apply it on the mixup distribution. 
For self-containation, the proof is given in Appendix~\ref{sec:optimization-guarantee}:
\begin{theorem}
\label{thm:opt}
    Under the assumptions 1-4 in Appendix~\ref{sec:optimization-guarantee} and  assume the sample space $\mathcal{Z}$ is a convex set, each local function $f_i$ is nonconvex, global function $F$ takes the form of our formulation~\eqref{eq: rfdd-optimization}, running Algorithm~\ref{algorithm:1} with the choice of  $\tau=T^{1/4}$, $\eta=\frac{1}{4LT^{3/4}}$ and $\gamma=\frac{1}{\sqrt{T}}$ we have,
\begin{align*}
&\frac{1}{T} \sum_{t=1}^{T} \mathbb{E}\left[\left\|\nabla \Phi_{1 /(2 L)}\left(\boldsymbol{w}^{t}\right)\right\|^{2}\right]\\ \leq &O\Bigg(\frac{D_{\Lambda}^{2}}{T^{1 / 8}}+\frac{\sigma_{\lambda}^{2}}{m T^{1 / 4}}+\frac{G_{\lambda}^{2}}{T^{1 / 4}}\\
&+\frac{G_{w} \sqrt{G_{w}^{2}+\sigma_{w}^{2}}}{T^{1 / 8}}+\frac{D_{\mathcal{W}}\left(\sigma_{w}+\sqrt{\Gamma}\right)}{T^{1 / 2}}\Bigg).
\end{align*}
\end{theorem}

% Moreover, we generalized our algorithm to include hard-constraints and mirror descent case and recover the same convergence rate as in the SGD.
% \begin{theorem}
%     Under similar assumptions and further assume each local function $f_i$ is convex, global function $F$ takes the form of our formulation~\eqref{eq: rfdd-optimization}, we have 
% \begin{align*}
%     \max _{\boldsymbol{\lambda} \in \Lambda} \mathbb{E}[F(\hat{\boldsymbol{w}}, \boldsymbol{\lambda})]-\min _{\boldsymbol{w} \in \mathcal{W}} \mathbb{E}[F(\boldsymbol{w}, \hat{\boldsymbol{\lambda}})] = O(\frac{1}{T^{3/8}}).
% \end{align*}
% \end{theorem}
\section{Experiments}
\label{sec:exp}
\begin{table*}[]
\centering
\caption{The average and worst-case RMSE of client test datasets (lower is better).}
\label{tab:overall_results}
\begin{tabular}{clllllll}
\hline
\multirow{2}{*}{Method}                     & \multicolumn{1}{c}{\multirow{2}{*}{Model}} & \multicolumn{2}{c}{ESOL}                 & \multicolumn{2}{c}{FreeSolv}             & \multicolumn{2}{c}{Drug Activity}             \\ \cline{3-8} 
                                            & \multicolumn{1}{c}{}                       & \multicolumn{1}{l}{Average} & Worst-case & \multicolumn{1}{l}{Average} & Worst-case & \multicolumn{1}{l}{Average} & Worst-case \\ \hline
\multirow{3}{*}{FedAVG}                     & GCN                                         & \multicolumn{1}{c}{$0.976_{(0.072)}$}         &$1.232_{(0.081)}$           & \multicolumn{1}{l}{$3.011_{(0.121)}$}        &$3.863_{(0.104)}$            & \multicolumn{1}{l}{$0.932_{(0.012)}$}        & $1.104_{(0.023)}$           \\ \cline{2-8} 
                                            & GAT                                         & \multicolumn{1}{l}{$1.134_{(0.043)}$}        & $1.343_{(0.075)}$          & \multicolumn{1}{l}{$2.641_{(0.431)}$}        & $3.456_{(0.321)}$            & \multicolumn{1}{l}{$0.964_{(0.035)}$}        &  $1.197_{(0.032)}$         \\ \cline{2-8} 
                                            & \multicolumn{1}{c}{MPNN}                   & \multicolumn{1}{l}{$0.923_{(0.067)}$}        & $1.312_{(0.065)}$           & \multicolumn{1}{l}{$2.974_{(0.669)}$}        & $3.689_{(0.278)}$           & \multicolumn{1}{l}{$0.931_{(0.032)}$}        &  $1.231_{(0.035)}$          \\ \hline
\multirow{3}{*}{DRFA}                       & GCN                                         & \multicolumn{1}{l}{$1.204_{(0.068)}$}        &$1.357_{(0.073)}$            & \multicolumn{1}{l}{$3.465_{(0.264)}$}        & $3.928_{(0.276)}$           & \multicolumn{1}{l}{$1.265_{(0.037)}$}        & $ 1.542_{(0.029)}$           \\ \cline{2-8} 
                                            & GAT                                         & \multicolumn{1}{l}{$1.346_{(0.065)}$}        &$1.466_{(0.046)}$            & \multicolumn{1}{l}{$2.897_{(0.568)}$}        & $3.699_{(0.772)}$           & \multicolumn{1}{l}{$1.323_{(0.027)}$}        &   $1.525_{(0.063)}$         \\ \cline{2-8} 
                                            & MPNN                                        & \multicolumn{1}{l}{$1.233_{(0.072)}$}        &$1.386_{(0.064)}$            & \multicolumn{1}{l}{$3.211_{(0.798)}$}        & $3.899_{(0.754)}$           & \multicolumn{1}{l}{$1.148_{(0.043)}$}        &  $1.458_{(0.028)}$          \\ \hline
\multicolumn{1}{l}{\multirow{3}{*}{Ours}} & GCN                                         & \multicolumn{1}{l}{$1.086_{(0.056)}$}        &$1.121_{(0.059)}$            & \multicolumn{1}{l}{$3.123_{(0.134)}$}        & $3.532_{(0.096)}$           & \multicolumn{1}{l}{$1.126_{(0.0480)}$}        & $1.065_{(0.013)}$           \\ \cline{2-8} 
\multicolumn{1}{l}{}                      & GAT                                         & \multicolumn{1}{l}{$1.243_{(0.049)}$}        & $1.234_{(0.055)}$           & \multicolumn{1}{l}{$2.801_{(0.654)}$}       & $3.207_{(0.441)}$           & \multicolumn{1}{l}{$1.021_{(0.065)}$}        & $1.132_{(0.028)}$            \\ \cline{2-8} 
\multicolumn{1}{l}{}                      & MPNN                                        & \multicolumn{1}{l}{$0.989_{(0.054)}$}        &$1.132_{(0.061)}$            & \multicolumn{1}{l}{$3.112_{(0.649)}$}        &  $ 3.579_{(0.489)}$          & \multicolumn{1}{l}{$0.987_{(0.041)}$}        &    $1.187_{(0.029)}$         \\ \hline
\end{tabular}
\end{table*}
\begin{figure*}
\centering
    \subfigure[Impact of different noise level.]{\includegraphics[width=0.4\textwidth]{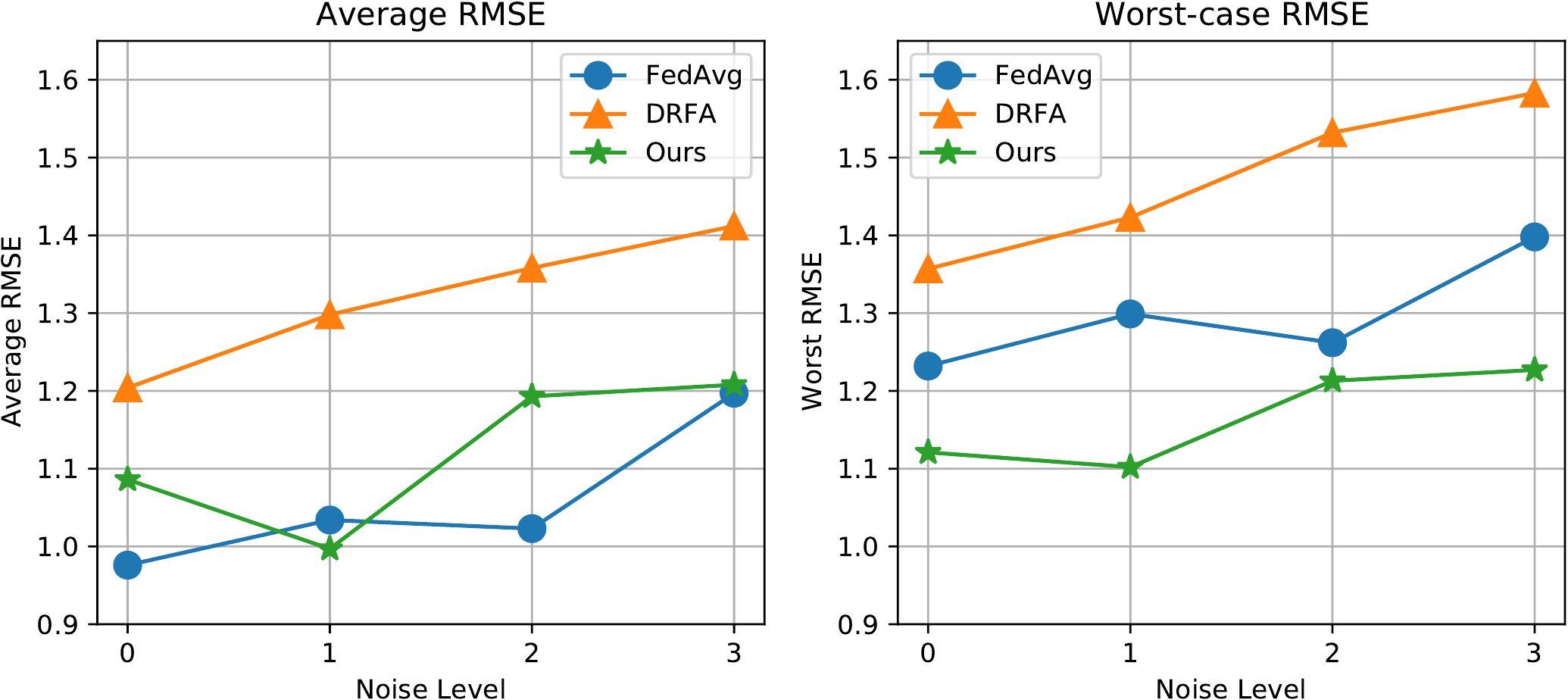}\label{fig:noise_level}}
     \subfigure[Different client number.]{\includegraphics[width=0.2\textwidth, height=0.17\textwidth]{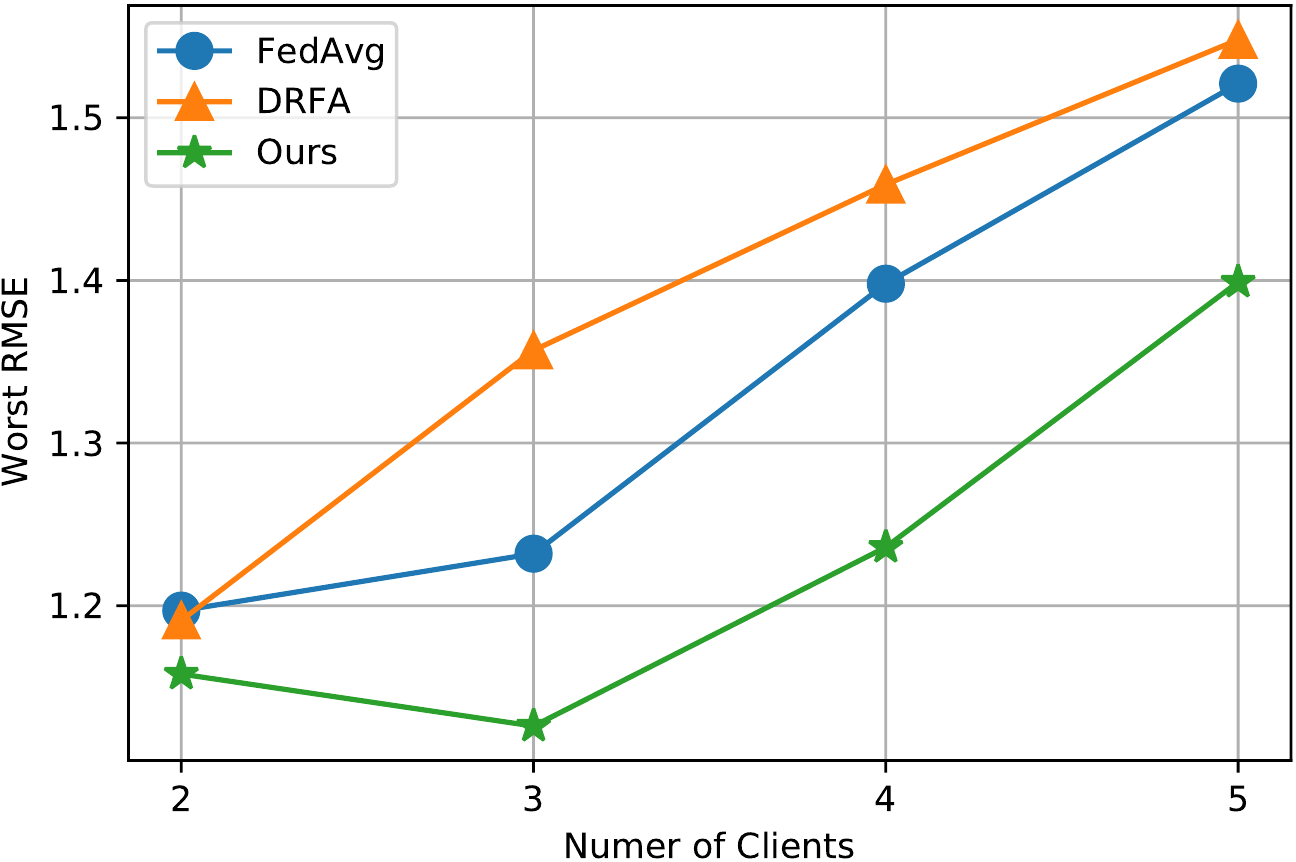}\label{fig:num_clients}}
     \subfigure[Imbalance mode.]{\includegraphics[width=0.2\textwidth, height=0.17\textwidth]{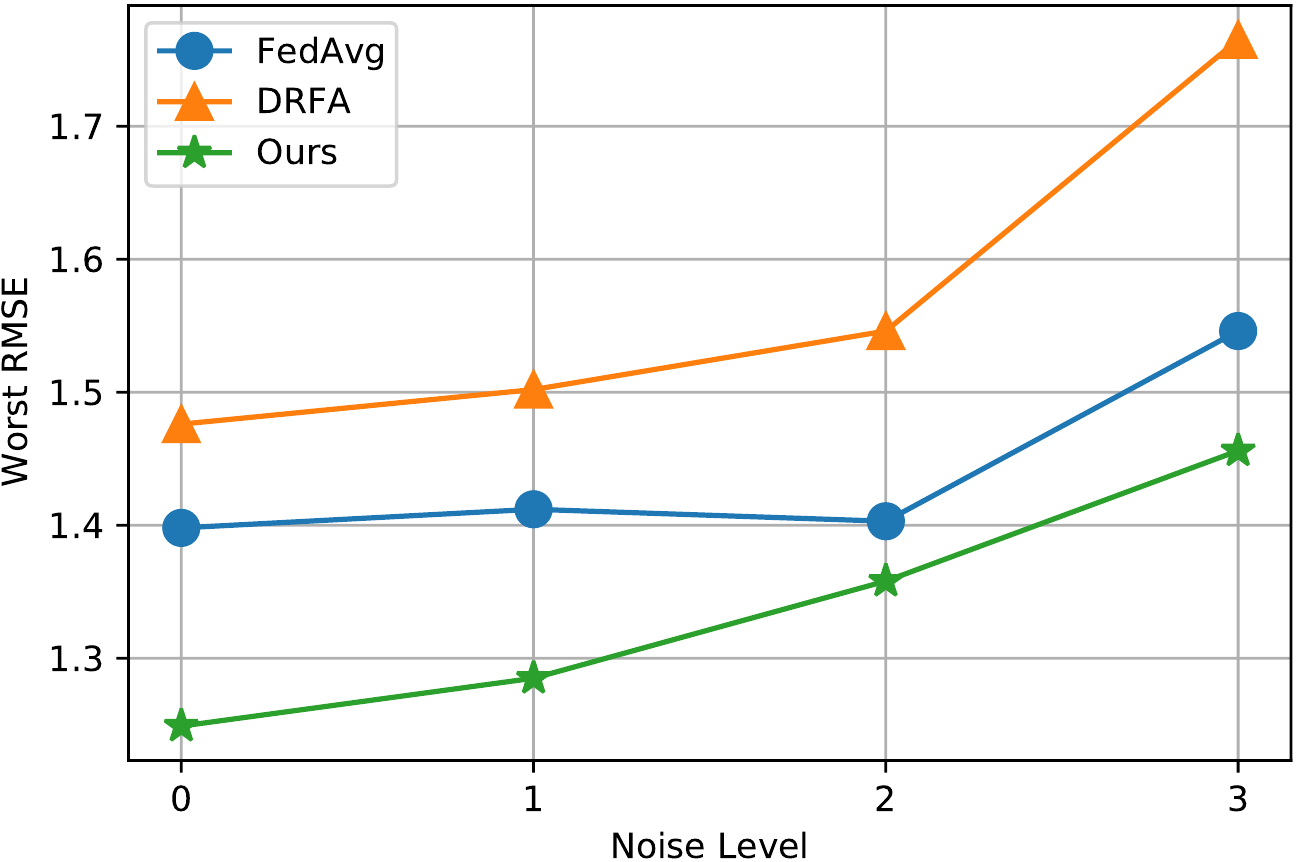}\label{fig:mode}}
     \caption{Evaluations under different settings including different noise levels, client number, and imbalance mode.}
\end{figure*}
Most previous work \cite{deng_distributionally_2021,mohri_agnostic_2019} focus on standard and small-scale academic benchmarks in which most samples are high-quality with negligible data noises. For example, one of the most related work to ours is DRFA \cite{deng_distributionally_2021}  which conducts experiments on Fashion MNIST. In this paper, we turn to evaluate our method on more challenging tasks in the drug discovery era. Except for drug discovery, we also provide additional results on other applications in the discussion part.
%Note that data noises are a notorious issue that widely exists in various drug discovery tasks. 
\subsection{Setup}

\nosection{Task and Dataset}
We consider two typical tasks in drug discovery, molecular property prediction and drug activity prediction. Specifically, molecular property
prediction aims to predict ADMET ((absorption, distribution, metabolism, excretion, and toxicity)) properties of a given molecular, which plays a key role in estimating drug candidates' safety and efficiency. For this task, we select two datasets from the MoleculenNet \cite{MolecularNet} with different properties, namely solubility, and lipophilicity. 
Drug affinity prediction dataset consists of $34000$ drug molecules and corresponding affinity degrees measured by the IC50 method from the ChEMBL website \citep{chembal}. The dataset comes from the 
DrugOOD project \citep{2022arXiv220109637J} which provides automated curator and benchmarks for the 
out-of-distribution problem in drug discovery.  

%aims to predict drug candidates' binding affinity against specific protein targets, which is a crucial task for hit finding in computer-aided and AI-aided drug discovery \cite{Sliwoski2013}. %\yt{will add the drugood reference}. 
%Our dataset . 
% \bz{yatao add some LBDD descriptions: done}

\nosection{Data Splitting}
A key step to evaluate different FL algorithms is data splitting, i.e, splitting the global dataset into different subsets and assigning them to different clients.
We consider the non-IID split in this paper. For molecular property prediction,
we employ scaffold splitting \cite{MolecularNet} based on scaffolds of molecular graph,  which ensures each client has different training distribution (varies in molecular scaffolds). 
For drug activity prediction, we employ assay splitting based on assay information, which ensures each client have data from different assay setting.
%For split ratio, we provide two additional modes, i.e., balance (each client has similar numbers of local samples ) and imbalance modes (each client varies in the number of local samples ).
Once the client splitting is finished, we split each local dataset into train/validation/test with a ratio $8:1:1$. We gather each local validation set into a global validation set for tuning various hyper-parameters in the training.

\nosection{Model and Training Setup}
In this paper, we select three mainstream GNN models for the above tasks including GCN \cite{GCN}, GAT \cite{gat} and MPNN \cite{MPNN}. The number of clients is set to $3$ in Table \ref{tab:overall_results}. We also provide results of 
other clients number in Figure \ref{fig:num_clients}.  We set the synchronization gap in Algorithm \ref{algorithm:1} to $5$ and set the local learning rate  to $0.01$. The total training round is set to $100$ for all datasets.

\nosection{Evaluation Metric}
We evaluate \rfdd  in terms of \emph{average test RMSE} and \emph{worst-case test RMSE}. Average test RMSE denotes the average RMSE of the global model over all local test datasets. This metric can reflect the average model performance
on different data distributions. While the worst-case test RMSE is the worst (highest) RMSE among all local test datasets. This metric can partially reflect the model's generalization ability on different data distributions. For all experiments, as suggested by prior work \cite{MolecularNet}, we apply three independent runs on three random seeded scaffold/assay splitting and report the mean and standard deviations.
\subsection{Results}
The overall results are shown in Table \ref{tab:overall_results}. As shown in the table, our method outperforms the previous state-of-the-art method (DRFA \cite{deng_distributionally_2021}) on all three datasets in terms of the worst-case RMSE among all clients. For example, in comparison with DRFA, our method has achieved
$17\%$ relative reduction of worst-case RMSE on the ESOL dataset owing to the use of mixup (from $1.357$ to $1.121$). We further plot the loss curve of the total training phase in Figure \ref{fig:convergence} which also shows our method's superiority in convergence speed.
Our method can also outperform the conventional method FedAVG in terms of the worst-case RMSE (1.121 vs 1.232). However, for average RMSE, the performance of our method is slightly worse than FedAVG. 
%Besides performance, our method can partially reduce the variance of most cases with different random seeds in terms of the standard deviations reported in Table~\ref{tab:overall_results}. This might be caused by that the mixup strategy can reduce the noise effect.
Despite the ESOL dataset, our method also shows its superiority in Freesolv and drug activity datasets (see more details in Table \ref{tab:overall_results}). 
%Among all these three datasets, xxx.

\nosection{Effects of label noise} From the result, we also observe that
conventional FedAVG method consistently outperforms  DRFA in terms of both average and worst-case accuracy. This might be caused by the label noise contained in local datasets.  To further validate this conjecture, we run three FL methods (FedAvg, DRFA, \rfdd ) on noisy datasets whose labels are perturbed by Gaussian noises with different magnitudes. Specifically, we randomly select $50\%$ clients and inject different levels (controlled by standard deviation) of Gaussian noises into their data labels with probability $30\%$.
As shown in Figure \ref{fig:noise_level}, along with the increase of noise level \footnote{larger level means the injected noises have larger standard deviation}, the worst-case RMSE of DRFA significantly drops which indicates that label noises can easily affect the performance of DRFA. The reason is that DRFA pessimistically optimizes the worst-case combination of local distribution which is typically dominated by the local training distribution with more noises. As a comparison, our method can boost the performance by combining DRFA with the mixup strategy in terms of both average and worst-case RMSE. Note that our method can also
achieve better worst-case RMSE than FedAvg. This indicates our method has better generalization ability on different local distributions than conventional FL methods which are appealing to real-world Non-IID settings.

\nosection{Number of clients}
The number of clients is set to $3$ in Table \ref{tab:overall_results}. We also provide results on other client number in Figure \ref{fig:num_clients}. As shown in the figure, our method consistently outperforms the other two methods in all settings in terms of the worst-case RMSE. And in most cases, the increase of the client number leads to significant
performance reduction. This may be caused by that we use non-iid splitting in this paper and the increase of client number can increase the model update estimation bias \cite{fedgnn}. We also note that DRFA achieves similar worst-case RMSE to FedAvg. One possible reason for this might be the decrease in client numbers can reduce the effect of local noises. 

\nosection{Imbalance mode}
Previous experiments are all conducted in the balance mode, i.e., the number of samples of each client is the same. In this part, we further evaluate our method under imbalance mode, i.e., the numbers of samples are highly unbalancing among all clients.
Specifically, we set the number of clients to 3 and let the sample ratio among these three clients be $7:2:1$. Under the imbalance mode, as shown in Figure \ref{fig:mode} we observe that the performance of DRFA becomes worse in contrast to the balanced mode. The situation becomes worse with the increase of noise level, which is due to that the client with minimum sample size magnifies the effect of local data noise on the final learned model. Also, our method can still reduce the noise effect under the imbalance mode.

\begin{figure}
    \centering
    \includegraphics[width=0.3\textwidth]{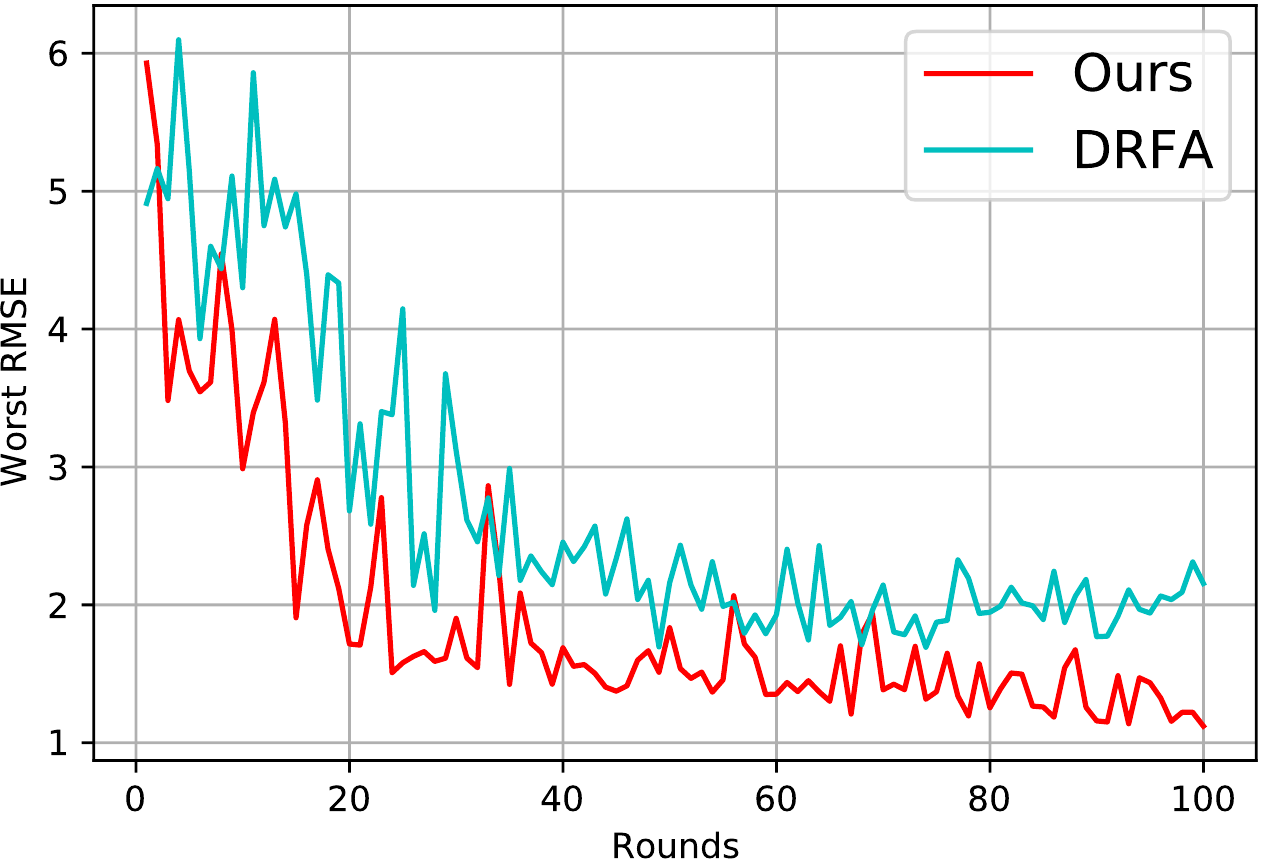}
    \caption{Convergence comparison between ours and DRFA.}
    \vspace{-0.5cm}
    \label{fig:convergence}
\end{figure}

%\nosection{Comparison with other FL methods}
%As discussed in related works, there are also other FL methods.

\nosection{Other application}
In the above discussions, we evaluate our method on challenging drug discovery tasks. To give a more comprehensive evaluation of our method, we further conduct experiments on some academic benchmarks including Fashion MNIST and UCI Adults, which are commonly used in previous work \cite{deng_distributionally_2021,mohri_agnostic_2019}. Following the setting of prior work~\cite{mohri_agnostic_2019}, we extract the subset of the data labeled with different labels and split this subset into three local datasets, and assign them to different clients.
For Fashion MNIST, the subsets are set to t-shirt/top, pullover, and shirt following the prior work \cite{mohri_agnostic_2019}. For the Adult dataset, we set subsets to doctorate and non-doctorate. The training round is set to $1000$, the overall results are shown in Table \ref{tab:others}. From this table, we observe that DRFA
outperforms all other two methods on the FMNIST dataset in terms of worst-case accuracy. One possible explanation is that the Fashion MNIST dataset contains high-quality labeled samples thus there is no noise that can affect the training procedure. This can be verified by conducting experiment on the "flipped" dataset, i.e., flipping the training samples' label with probability $30\%$. For the Fashion MNIST with label flipping, the performance of DRFA becomes worse than the others methods including ours. The performance drop can be caused by 
the injected flipping noise. Moreover, our method can significantly boost the worst-case classification accuracy by using mixup to reduce the noise effect. 
\begin{table}[]
\caption{The worst-case classification accuracy of FMNIST and Adult datasets. Flip means flipping the training samples' label with a given probability.}
\label{tab:others}
\centering
\begin{tabular}{cccc}
\hline
Dataset & FedAvg & DRFA & Ours \\ \hline
FMNIST  & $70.64_{(0.212)}$       &  $71.04_{(0.321)}$    &$70.89_{(0.234)}$       \\ \hline
Adult   &  $68.21_{(0.187)}$      & $67.31_{(0.256)}$     &$69.87_{(0.268)}$      \\ \hline
FMNIST+Flip &$67.21_{(0.342)}$  &$66.28_{(0.439)}$ &$69.78_{(0.376)}$ \\ \hline
Adult+Flip &$66.23_{(0.232)}$ &$64.86_{(0.314)}$ &$68.66_{(0.321)}$\\ \hline
\end{tabular}
\vspace{-0.5cm}
\end{table}
% \usepackage{multirow}

\section{Related Works}
\subsection{ Distributionally Robust Optimization}
Distributionally robust optimization is a general learning paradigm that enforcing
the learned model to perform well over the local worst-case risk functions. Generally,
the local worst-case risk function refers to the supremum of the risk function
over \emph{ambiguity set} which is a vicinity of the empirical data distribution \cite{duchi2011adaptive}. Data distributions contained in the ambiguity set have small distances
to the true empirical data distribution wherein the distance is computed based on
well-known probability metrics such as $\phi$-divergence and wasserstein distance \cite{wang_mixup_2021}. 
Although DRO can be always formulated as a minimax optimization problem, in most real-world cases, exactly solving this minimax problem is intractable since it involves the supremum over infinite many probability distributions, thus two streams of literature to handle it by using either a primal-dual pair of infinite-dimensional linear programs \cite{esfahani2018data} or first-order optimality conditions of the dual. 
Amongst them, the most related to us is \cite{namkoong2016stochastic}, which proposes an efficient bandit mirror descent algorithm for the minimax problem when the uncertainty set is constructed using $f$-divergence.
% Motivated by the above work, in this paper, we combine DRO with local data mixup in the FL setting to minimize the worst-case combination of local empirical losses while reducing the effects of intra-client data noises.

\subsection{Federated Learning on Non-IID Data}
%Conventional FL method like FedAvg\cite{fedavg} has achieved remarkable performance over IID data from different clients. However, when applying FedAvg to Non-IID data, it suffers slow and unstable convergence while the learned global model has poor performance.
There is a rich body of literatures aims to modify the conventional FL method for solving inter-client data heterogeneity (Non-IID issue). One typical way to improve FedAvg on the Non-iid setting is to employ different local training strategies for different clients \cite{Li2020On,Fedprox,reddi2021adaptive}. We call this line of work \emph{adaptive federated learning} since they need to either adaptively adjust the local learning rate or modify the local empirical loss function by adding a client-aware regularizer. For example, 
%a previous work \cite{reddi2021adaptive} proposes an adaptive optimizer for local model updating (e.g., Adam and AdaGrad ) which allows to implicitly adjust local learning rate according to the loss function computed based on local data.
Scaffold~\cite{Scaffold} identifies that the performance degradation of FedAvg on Non-iid data is caused by \emph{"client drift"} issue. and introduces extra device variables to adaptively adjust the local loss.
%Another research line is personalized FL which first trains a global model then 
%fine-tune it using local client data.
%Prior work \cite{fed_non_iid_graph} points out that this type of work may fail
%when local data with high heterogeneity. And it is also inefficient to maintain an individual local model for each client.
 A more general and essential way to overcome the client  heterogeneity is to incorporate distributionally robust optimization (DRO) \cite{kwon_principled_2020,zhai_doro_2021,duchi2011adaptive} into conventional FL paradigm \cite{mohri_agnostic_2019,deng_distributionally_2021}.  
DRFA \cite{deng_distributionally_2021} is a recent work that proposes to optimize the agnostic (distributionally robust) empirical loss, which 
combines different local loss functions with learnable weights. This method theoretically ensures that the learned model can perform well over the worst-case local loss combination. Previous work in this line has shown its effectiveness in deep neural networks on some small-scale benchmarks (e.g., Fashion MNIST).

\vspace{-0.5em}
\section{Conclusion}
In this paper, we propose a general framework which simultaneously mitigate the inter-client data heterogeneity (i.e., Non-IID issue) and the intra-client data noise. We provide comprehensive theoretical analysis of the proposed approach. We also apply our method to real-world 
drug discovery tasks and show the superiority of our method in reducing the local noise effects. Our results show: (1)Directly applying DRO methods to FL settings suffers poor performance caused by intra-client data noise. (2) Combining mixup strategy with DRO can mitigate the noise side effects. Despite FL, our framework can be employed in other scenarios, e.g., fair machine learning. The future work includes improving the theoretical properties of mixup in the FL setting. Another interesting direction would be exploring the cross-client mixup strategy (i.e., mixup data from different clients). 
%%%%%%%%%%%%%%%%%%%%%%%%%%%%%%%%%%%%%%%%%%%%%%%%%%%%%%%%%%%%%%%%%%%%%%%%%%%%%%%
%%%%%%%%%%%%%%%%%%%%%%%%%%%%%%%%%%%%%%%%%%%%%%%%%%%%%%%%%%%%%%%%%%%%%%%%%%%%%%%

\bibliography{main}
\bibliographystyle{icml2022}
\appendix
\onecolumn
\icmltitle{Appendix For DRFLM: Distributionally Robust Federated Learning \\
with Inter-client Noise via Local Mixup}

In this appendix, 
Section~\ref{sec: appendix counter example} gives an detailed exploration on Example~\ref{example: exmp1} to illustrate why DRFA may fail in the federated learning setting while our proposed~\rfdd can rescue.
Section~\ref{sec: appendix generalization bound} provide the formal proof for the Theorem~\ref{thm: Generalization guarantee}.
Section~\ref{sec:optimization-guarantee} presents the proof for our~\rfdd's optimization guarantee (Theorem~\ref{thm:opt}) and the associated assumptions.
Section~\ref{appendix:non-linear} is an extension for Theorem~\ref{thm: Generalization guarantee}.

\section{Proofs for Counter-Examples}\label{sec: appendix counter example}
\subsection{One-Dimensional Example of ERM}\label{subsec-1d-example}
To prove our statements in Example~\ref{example: exmp1}, we first study the following one-dimensional example: 
Consider $P_{x,y}$ defined as following: \\
    i) The margin distribution $P_x$ is uniform on intervals $[-2,-1] \cup [1,2]$\\
    ii) Condition on $x$, $P(y\lvert x) =   1 - 2 \cdot \bm{1}\{x\leq  0\} $\\
Consider the ERM procedure over the function classes $\mathcal{F}: = \{f_x(z): = 1 -2 \cdot \bm{1}\{z< x\} \}$, we have in noiseless case the excess risk of the ERM estimator will converge to $0$ as sample-size increasing. Similar as in Example~\ref{example: exmp1}, for $1/2<p_1<1$ we define $I = [-2,-1]$, and suppose the random perturbation is given in form of flipping labels for $x\in I$ with probability $p_1$, then as sample-size increasing there will be $p_1$-ratio of samples has label $1$ in the interval $I$. As a result, we have then the decision boundary of ERM estimator will turns to the left-boundary point of $I$, i.e. $-2$, thus will have excess risk $1/2>0.$ While if we consider the all-sample mix-up with constant mixture factor $1/2$, then as sample-size increasing the mixed up sample turns to distributed as following:\\
   i) $1/4$ ratio samples lies between $[1,2]$ with label $1$,\\
     ii) $1/2$ ratio samples lies between $[-1,1]$ with label $\dfrac{1}{2}$ with probability $1-p_1$ and $1$  with probability $p$,\\
     iii) $1/4$ ratio samples lies between $[-2,-1]$ with label $1$ with probability $p^2$ ,$\dfrac{1}{2}$ with probability  $2p(1-p)$  and $0$ with probability $(1-p)^2$.\\
Now a classifier with decision boundary at $x_0= -2$ will incur the least squared error$\frac{1}{4}(1-p) + \frac{1}{2}p(1-p) + \frac{1}{4}p^2 = \frac{1}{2}$
while a classifier with decision boundary at $x_0 = -1 $ will incur the least squared error $\dfrac{1}{4}p^2+\dfrac{1}{2}p(1-p)<1/2$ whenever $p>1/2$.
Thus the ERM estimator over the mixed up distribution will select its decision boundary at the right-hand-side of $-1$, thus gives $0$ population risk.

\subsection{Proof of Statements in  Example~\ref{example: exmp1}}
Now we consider the Example~\ref{example: exmp1}:
% \textbf{our point is that if some client has a large noise in his local dataset, then the performance of global model trained by DRFA scheme is bad.}  
Firstly, notice that in this example, $\Lambda = \{\bm\lambda\in\mathbb{R}^2_+,\sum_{i=1}^2 \lambda_i = 1 \}$, thus the empirical FedAvg, empirical DRFA loss, and the empirical \rfdd loss and  are given by $\frac{1}{2}(f_1(\w) + f_2(\w) ) ,\max_{i\in \{1,2\}} f_i(\w),\max_{i\in \{1,2\}} \tilde{f}_i(\w)  $. Noticing that as sample-size increasing, we have for $\w$ with decision boundary of $f(\x;\w)$ lies in $-2<b<2$ will incur the following empirical loss: \begin{align*}
    f_1(\w) &= -\dfrac{1+b}{2}\cdot  \bm{1}\{ b< - 1 \} + \dfrac{b-1}{2}\cdot \bm{1}\{ b> 1 \}.\\
    f_2(\w) &= ( \dfrac{(1-p_1)p_2}{2} - \dfrac{1 + p_2 + b}{2} )\cdot  \bm{1}\{  b<-1-p_2 \} + (\dfrac{b-1}{2}  + \dfrac{p_1 p_2}{2} )\cdot \bm{1}\{ b> 1 \}  ... \\
    &+( -\dfrac{b+1}{2}(1-p_1) + \dfrac{b+1+p_2}{2}p_1  )  \bm{1}\{ b\in I \} .
\end{align*}
Thus by elementary calculation, we have as ${\tilde{N}}\to\infty$, the decision boundary corresponding to $f(\x;\w_{\tilde{N}}^{\text{DRFA}})$ will turn to $-1-p_2.$ On the other hand, we can show that the the decision boundary corresponding to $f(\x;\w_{\tilde{N}}^{\text{Avg}})$ can be  any point in  $[-1,1]$. 

Finally, 
analyze the decision boundary of $f(\x;\w_{\tilde{N}}^{\text{\rfdd}})$, we first assume as in example in section \ref{subsec-1d-example} that the mix-up of two data $\x,\x'$ is given by $\dfrac{1}{2}(\x +\x')$, such simplification will not loss the generality because it equals to the in-expectation value of mixed-up data when $\alpha = \beta.$ Firstly, notice that for client $1$, we have obviously that for $\w_1$ the empirical-risk-minimizer of $ \tilde{f}_1(\w),$ its decision boundary lies between $[-1,1]$. For client $2$, we have when looking at the first coordinate, there are fewer samples with label $-1$ when $x_1\leq 1$ than the one-dimensional example in section~\ref{subsec-1d-example}, thus the decision boundary of corresponding empirical risk minimizer of $\tilde{f}(\w)$ will in the right hand side of the decision boundary of mixed-up ERM considered in section~\ref{subsec-1d-example}, i.e. it will in the right hand side of $-1$. That shows $\tilde{f}_2(\w)$ also encourage its minimizer to lie in $[-1,2]$. As a consequence, we have the decision boundary of the ERM of $\max_{i\in \{1,2\}} \tilde{f}_i(\w)$ will lie in $[-1,2] $ as ${\tilde{N}}\to\infty$, which will has zero population risk.

\section{Proofs for Generalization Guarantee }\label{sec: appendix generalization bound}

\subsection{Mix-Up Effect for Generalizaed Linear Model}
Using the Lemma~3.3 in \cite{zhang_how_2021}, we have for $f(\w;\x,y) = \mu(\w^T \x) - y\w^T \x$ and $f_j(\w) = \dfrac{1}{N_j}\sum_{i=1}^{N_j} f(\w;x_j^i,y_j^i)$, when $\sum_{j} x_j^i = 0$, the second order approximation of \rfdd loss is given by \begin{align*}
	 \min_{\w \in W}\max_{\lambda\in \Lambda} \sum_{i=1}^N \lambda_i ({f}_i(\w)+   \underbrace{\dfrac{c}{2} \zeta_i(\w) \w^T\hat\Sigma_i \w   }_{R_i(\w)} )
\end{align*} 
with $c = \E_{\gamma \sim D_\gamma}[\dfrac{(1-\gamma)^2}{\gamma^2}] $ and $D_\gamma = \dfrac{\alpha}{\alpha+\beta} \text{Beta}(\alpha + 1,\beta)+\dfrac{\alpha}{\alpha + \beta} \text{Beta}(\beta + 1, \alpha).$ Now recall the assumption that $K^{-1}\leq  \lvert \mu''(z)\rvert \leq K$, we get then $\zeta_i(\w) \geq K^{-1}$ , thus for any $\lambda \in \Lambda$ we have
\begin{align*}
	 \sum_{i=1}^N \lambda_i \dfrac{c}{2} \zeta_i(\w) \w^T\hat\Sigma_i \w \geq  \min_{\lambda \in \Lambda} \sum_{i=1}^N \lambda_i \dfrac{c}{2K}  \w^T\hat\Sigma_i \w .
\end{align*}
By $\lVert x_{j}^i\rVert_2,\w\leq 1$, we have for $\hat{\w}$ the minimizer of \rfdd, \begin{align*}
	 \sum_{i=1}^N \lambda_i \dfrac{c}{2K}  \hat{\w}^T\hat\Sigma_i \hat{\w} \leq \dfrac{cK}{2}.
\end{align*} 
i.e. $\hat{w} \in \mathcal{W}_r$ with $r= K^2$.

\subsection{Proof of Theorem~\ref{thm: Generalization guarantee}}

We first establish the following standard uniform deviation bound result based on the Rademacher complexity, which is an analogue of Theorem 10 of \cite{mohri_agnostic_2019}: \begin{lemma} Assuming the loss function $\ell$ is bounded by $M$. Fix $\epsilon>0$ and $\mathbf{n}: = (n_1,\dots,n_m)$. Then for any $\delta>0$, with probability at least $1-\delta$ over the draw of samples $D_j\sim P_j^{N_j}$, the following inequality holds for all $\phi \in\mathcal{F}$ and $\lambda\in\Lambda $: \begin{align*}
 \sum_{j=1}^m \lambda_j\E_{P_j}[\ell(f(x),y)]&\leq \sum_{j=1}^N \dfrac{\lambda_j}{N_j}\sum_{i=1}^{N_j}\ell(\phi (x_i^j),y_i^j)\\
 &+\sum_{j=1}^N\lambda_j\big(  \mathcal{R}^j_{N_j}(\mathcal{F}) +M\sqrt{\dfrac{\log(N/\delta)}{2N_j}}   \big)
\end{align*}
where \begin{align*}
    \mathcal{R}^j_{N_j}(\mathcal{F}) =\E_{D_j}\big[\E_{\varepsilon}[\sup_{\phi \in \mathcal{F}}\dfrac{1}{N_j}\sum_{i=1}^{N_j}\varepsilon_{ij}\ell(\phi (x_i^j,y^j_i))]\big].
\end{align*}
\end{lemma}
\begin{proof}[Proof of the Lemma] Recall that by Theorem~8 of \cite{Bartlett2002}, for any $M$-uniformly bounded and $L$-Lipschitz function $\ell$, for all $\phi \in \mathcal{F}$, with probability at least $1-\delta,$ \begin{align*}
	\E[\ell(\phi (x),y)]\leq \dfrac{1}{n}\sum_{i=1}^n\ell(\phi (x_i),y_i) + 2L \mathcal{R}_{n}(\mathcal{F})+ M \sqrt{\dfrac{\log(1/\delta)}{2n}}
\end{align*}
	Now applying this result to each client $j$ and use the union bound leads to the desired result.
\end{proof}

\begin{proof}[Proof of Theorem~\ref{thm: Generalization guarantee} ]

To apply the lemma, we need to compute the empirical Rademacher complexity $\mathcal{R}_{N_j}^j(\mathcal{F};D_j)$ firstly: \begin{align*}
\mathcal{R}_{N_j}^j(\mathcal{W}_r;D_j) & := \E_{\varepsilon}[\sup_{w\in \mathcal{W}_r}\dfrac{1}{N_j}\sum_{i=1}^{N_j}\varepsilon_{ij} \w^T x_i^j]\\
	& = \E_{\varepsilon}[\sup_{\w^T\bar{\Sigma}(\w) \w \leq r}\dfrac{1}{N_j}\sum_{i=1}^{N_j}\varepsilon_{ij} \w^T x_i^j  ]\\
	& = \E_{\varepsilon}[\sup_{\w^T\bar{\Sigma}(\w) \w\leq r}\dfrac{1}{N_j}\sum_{i=1}^{N_j}\varepsilon_{ij}  \w^T\bar{\Sigma}(\w)^{1/2}\big(\bar{\Sigma}(\w)^{1/2}\big)^\dagger   x_i^j  ]\\
	&\leq r \E_{\varepsilon}[\sup_{\w^T\bar{\Sigma}(\w) \w\leq r}\dfrac{1}{N_j}\big\lVert \sum_{i=1}^{N_j}\varepsilon_{ij}  \big(\bar{\Sigma}(\w)^{1/2}\big)^\dagger   x_i^j \big\rVert  ]\\
	&\leq \dfrac{r}{{N_j}}\sqrt{ \E_{\varepsilon}[\sup_{\Sigma \in \Sigma_\Lambda }\big\lVert \sum_{i=1}^{N_j}\varepsilon_{ij}  \Sigma^\dagger   x_i^j \big\rVert^2 ]}\\
	&\leq \dfrac{r}{{N_j}}  \sqrt{ \sup_{\Sigma \in \Sigma_\Lambda } \sum_{i=1}^{N_j}  (x_i^j )^T\Sigma^{\dagger}    x_i^j  }
\end{align*}
Now noticing \begin{align*}
    \mathcal{R}_{N_j}^j(\mathcal{W}_r) &= \E_{D_j}[\mathcal{R}_{N_j}^j(\mathcal{F};D_j)]\\
    &\leq \dfrac{r}{\sqrt{N_j}}\cdot \sqrt{\E[\dfrac{1}{N_j}  \sup_{\Sigma \in \Sigma_\Lambda } \sum_{i=1}^{N_j}  (x_i^j )^T\Sigma^{\dagger}    x_i^j  ]}
\end{align*} and applying Lemma~1 leads to the desired result.
\end{proof}

\subsection{Extending to Non-linear case}
\label{appendix:non-linear}
As in \cite{zhang_how_2021}, our analysis for generalized linear model can also be extended to the second-layer neural network with squared loss: In that case $f(\w;\x,y) = \big(y-\theta_1^T\sigma(W\x)-\theta_0\big)^2 $ with  $\theta_1\in \R^p, \theta_0\in\R, W\in \R^{p\times d} $ and $ \w$ consists of  $\theta_1,\theta_0,W)$ . If we perform mixup on the second layer and assume without loss of generality that $\{\sigma(Wx^j_i)\}_{i=1}^{N_j}$ are centered, then by Lemma~3.4 of \cite{zhang_how_2021}, we have the second order approximation of \rfdd~ loss is given by  \begin{align*}
	 \min_{\w \in W}\max_{\lambda\in \Lambda} \sum_{i=1}^N \lambda_i ({f}_i(\w)+   {\dfrac{c}{2} \zeta_i(\w) \w^T\hat\Sigma_i^\sigma \w   } ),
\end{align*} 
where $\hat\Sigma_i^\sigma$ is the sample covariance of $\sigma(Wx^j_i)$ of $i$-th client. Using the same argument as in GLM case, we have such loss forces its minimizer lies in \begin{align*}
	\mathcal{W}_r^\sigma: =  \{ \x\to\theta_1^T \sigma(W\x) + \theta_0: \min_{\lambda \in \Lambda} \sum_{i=1}^N \lambda_i  \theta_1^T\hat{\Sigma}_i^\sigma \theta_1 \leq r \}
\end{align*}
for some $r$.
To study the generalization ability, notice that by the similar argument as in generalized linear setting, we have \begin{align*}
    \mathcal{R}^j_{N_j}(\mathcal{W}_r^\sigma)\leq \dfrac{r}{\sqrt{N_j}}\cdot \sqrt{ \E[\sup_{\Sigma \in \Sigma^\sigma_\Lambda} \sum_{i=1}^{N_j}\text{tr}( \Sigma^\dagger \hat\Sigma^\sigma_j )   ] }.
\end{align*}
Where $\Sigma^\sigma_\Lambda = \{\sum_{j=1}^N \lambda_i \hat{\Sigma}_j^\sigma :\lambda \in \Lambda  \}$.
Now apply the Lemma~1 leads to the following formal result: \begin{theorem}[Generalization Bound of $\mathcal{W}_r^\sigma$] Suppose both $\theta_1,W,\theta_0$ are  bounded,
then there exists constants $L,B>0$ such that  for  $\w\in\mathcal{W}_r^\sigma$, the following inequality holds  with probability $1-\delta$:
\begin{align*} 
&\max_{\lambda\in\Lambda}    \sum_{j=1}^N\lambda_j\E_{P_j(\x,y)}[f(\w;\x,y)]\\
&\leq  \max_{\lambda\in\Lambda} \sum_{j=1}^N \dfrac{\lambda_j}{N_j}\sum_{i=1}^{N_j} f(\w;x_j^i,y_j^i) + \max_{\lambda\in\Lambda}\sum_{j=1}^N \lambda_j\sqrt{\dfrac{1}{N_j} }\big(\sqrt{\log\dfrac{N}{\delta} }\\
&+\sum_{j=1}^N rL\cdot \textcolor{red}{\sqrt{\big(H_j^\sigma\big)/N_j}} \big).
\end{align*} with 
 $H_j^\sigma: = \E_{P_j}[ \max_{\Sigma \in \Sigma^\sigma_\Lambda} \text{tr}\big(\Sigma^{\dagger}\hat{\Sigma}_j^\sigma\big)] 
$
\end{theorem}

\section{Optimization Guarantee}
\label{sec:optimization-guarantee}
\begin{assumption}[Weighted Gradient Dissimilarity]
\label{as:WGD}
A set of local objectives $f_{i}(\cdot), i=1,2, \ldots, N$ exhibit $\Gamma$ gradient dissimilarity defined as $\Gamma:=\sup_{\boldsymbol{w} \in \mathcal{W}, \blambda \in \Lambda, i \in[n]}, \sum_{j \in[n]} \lambda_{j}\left\|\nabla f_{i}(\boldsymbol{w})-\nabla f_{j}(\boldsymbol{w})\right\|^{2} .$
\end{assumption}

\begin{assumption}
\label{as:GB}
The gradient w.r.t $\boldsymbol{w}$ and $\boldsymbol{\lambda}$ are bounded, i.e., $\left\|\nabla f_{i}(\boldsymbol{w})\right\| \leq$ $G_{w}$ and $\left\|\nabla_{\boldsymbol{\lambda}} F(\boldsymbol{w}, \boldsymbol{\lambda})\right\| \leq G_{\lambda}$.
\end{assumption}

\begin{assumption}
\label{as:BD}
The diameters of $\mathcal{W}$ and $\Lambda$  are bounded by  $D_{\mathcal{W}}$  and  $D_{\Lambda}$.
\end{assumption}

\begin{assumption}
\label{as:BV}
Let $\hat{\nabla} F(\boldsymbol{w} ; \boldsymbol{\lambda})$ be a stochastic gradient for $\boldsymbol{\lambda}$, which is the $N$-dimensional vector such that the $i$-th entry is $f_{i}(\w ; \boldsymbol{z})$, and the rest are zero. 
Then we assume $\left\|\nabla f_{i}(\boldsymbol{w} ; \boldsymbol{z})-\nabla f_{i}(\boldsymbol{w})\right\| \leq \sigma_{w}^{2}, \forall i \in[N], \boldsymbol{z}\in \mathcal{Z}$ and $\|\hat{\nabla} F(\boldsymbol{w} ; \boldsymbol{\lambda})-\nabla F(\boldsymbol{w} ; \boldsymbol{\lambda})\| \leq \sigma_{\lambda}^{2}$. 
\end{assumption}

\begin{lemma}[One iteration analysis \cite{deng_distributionally_2021}]
\label{one-iteration}
Under the assumptions \ref{as:WGD}-\ref{as:BV}, the following statement holds for algorithm~\ref{algorithm:1}:
$$
\begin{aligned}
\mathbb{E}\left[\Phi_{1 / 2 L}\left(\bar{\boldsymbol{w}}_{t}\right)\right] \leq \mathbb{E}\left[\Phi_{1 / 2 L}\left(\bar{\boldsymbol{w}}_{t-1}\right)\right]+2 \eta D_{\mathcal{W}} L^{2} \mathbb{E}\left[\frac{1}{m} \sum_{i \in \mathcal{D}^{\lfloor \frac{t-1}{\tau}\rfloor}}\left\|\boldsymbol{w}^{(i)}_{t-1}-\bar{\boldsymbol{w}}_{t-1}\right\|\right] \\
+2 \eta L\left(\mathbb{E}\left[\Phi\left(\bar{\boldsymbol{w}}_{t-1}\right)\right]-\mathbb{E}\left[F\left(\boldsymbol{w}_{t-1}, \boldsymbol{\lambda}_{\lfloor\frac{t-1}{\tau}\rfloor}\right)\right]\right)
-\frac{\eta}{4}\mathbb{E}\left[\left\|\nabla \Phi_{1 /(2 L)}\left(\bar{\boldsymbol{w}}_{t-1}\right)\right\|^{2}\right] .
\end{aligned}
$$
\end{lemma}

\begin{lemma}[Bounded Norm Deviation) \cite{deng_distributionally_2021}]
\label{lm:bounded-norm-derivation}
$\forall i \in C_{\lfloor\frac{t}{\tau}\rfloor}$, the norm distance between $\bar{\boldsymbol{w}}_{t}$ and $\boldsymbol{w}^{(i)}_{t}$ is bounded as follows:
$$
\frac{1}{T} \sum_{t=0}^{T} \mathbb{E}\left[\frac{1}{m} \sum_{i \in C_{\lfloor\frac{t}{\tau}\rfloor}}\left\|\boldsymbol{w}^{(i)}_{t}-\bar{\boldsymbol{w}}_{t}\right\|\right] \leq 2 \eta \tau\left(\sigma_{w}+\frac{\sigma_{w}}{m}+\sqrt{\Gamma}\right) .
$$
\end{lemma}

\begin{lemma}[Lemma 9, \cite{deng_distributionally_2021}]
\label{lm:phi-F-2}
For algorithm~\ref{algorithm:1}, under the assumptions~\ref{as:WGD}-\ref{as:BV}, the following statement holds true:
$$
\frac{1}{T} \sum_{t=1}^{T}\left(\mathbb{E}\left[\Phi\left(\bar{\boldsymbol{w}}_{t}\right)\right]-\mathbb{E}\left[F\left(\bar{\boldsymbol{w}}_{t}, \boldsymbol{\lambda}_{\lfloor\frac{t}{\tau}\rfloor}\right)\right]\right) \leq 2 \sqrt{S} \tau \eta G_{w} \sqrt{G_{w}^{2}+\sigma_{w}^{2}}+\gamma \tau \frac{\sigma_{\lambda}^{2}}{m}+\gamma \tau G_{\lambda}^{2}+\frac{D_{\Lambda}^{2}}{2 \sqrt{S} \tau \gamma}.
$$
\end{lemma}

\textbf{Proof of Theorem~\ref{thm:opt}}
\begin{proof}
From lemma~\ref{one-iteration} we know,
\begin{align*}
\frac{1}{T} \sum_{t=1}^{T} \mathbb{E}\left[\| \nabla \Phi_{1 / 2 L}\right.&\left.\left(\bar{\boldsymbol{w}}_{t}\right) \|^{2}\right] \\
\leq & \frac{4}{\eta T} \mathbb{E}\left[\Phi_{1 /(2 L)}\left(\bar{\boldsymbol{w}}_{0}\right)\right]+\frac{D_{\mathcal{W}} L^{2}}{2 T} \sum_{t=1}^{T} \mathbb{E}\left[\frac{1}{m} \sum_{i \in C_{\lfloor \frac{t}{\tau}\rfloor}}\left\|\boldsymbol{w}^{(i)}_{t}-\bar{\boldsymbol{w}}_{t}\right\|\right] \\
&+\frac{L}{2 T} \sum_{t=1}^{T}\left(\mathbb{E}\left[\Phi\left(\bar{\boldsymbol{w}}_{t}\right)\right]-\mathbb{E}\left[F\left(\bar{\boldsymbol{w}}_{t}, \boldsymbol{\lambda}_{\lfloor\frac{t}{\tau}\rfloor}\right)\right]\right) .
\end{align*}
Plugging in Lemma~\ref{lm:bounded-norm-derivation} and~\ref{lm:phi-F-2} yields:
$$
\begin{aligned}
\frac{1}{T} \sum_{t=1}^{T} \mathbb{E}\left[\left\|\nabla \Phi_{1 / (2 L)}\left(\bar{\boldsymbol{w}}_{t}\right)\right\|^{2}\right] 
\leq & \frac{4}{\eta T} \mathbb{E}\left[\Phi_{1/(2 L)}\left(\bar{\boldsymbol{w}}_{0}\right)\right]+\eta \tau D_{\mathcal{W}} L^{2}\left(\sigma_{w}+\frac{\sigma_{w}}{m}+\sqrt{\Gamma}\right) \\
&+\sqrt{S} \tau \eta G_{w} L \sqrt{G_{w}^{2}+\sigma_{w}^{2}}+\gamma \tau \frac{\sigma_{\lambda}^{2} L}{2 m}+\gamma \tau \frac{G_{\lambda}^{2} L}{2}+\frac{D_{\Lambda}^{2} L}{4 \sqrt{S} \tau \gamma}.
\end{aligned}
$$
Plugging in $\eta=\frac{1}{4 L T^{3 / 4}}, \gamma=\frac{1}{T^{1 / 2}}$ and $\tau=T^{1 / 4}$ we obtain the convergence rate as cliamed:
$$
\begin{aligned}
\frac{1}{T} \sum_{t=1}^{T} \mathbb{E}\left[\left\|\nabla \Phi_{1 /(2 L)}\left(\bar{\boldsymbol{w}}_{t}\right)\right\|^{2}\right] \leq & \frac{4}{T^{1 / 4}} \mathbb{E}\left[\Phi_{1 / 2 L}\left(\boldsymbol{w}_{0}\right)\right]+\frac{L^{2}}{T^{1 / 2}}\left(\sigma_{w}+\frac{\sigma_{w}}{m}+\sqrt{\Gamma}\right) \\
&+\frac{1}{T^{1 / 8}} G_{w} L \sqrt{G_{w}^{2}+\sigma_{w}^{2}}+\frac{\sigma_{\lambda}^{2} L}{2 m T^{1 / 4}}+\frac{G_{\lambda}^{2} L}{2 T^{1 / 4}}+\frac{D_{\Lambda}^{2} L}{4 T^{1 / 8}}.
\end{aligned}
$$
Note that here the expectation takes over the randomness incurred by the mixup sampling.
\end{proof}

\end{document}